\documentclass[sigconf]{acmart}

\renewcommand\footnotetextcopyrightpermission[1]{}
\setcopyright{none}

\usepackage{amsmath,amsthm}

\usepackage{amssymb}
\usepackage{algorithm,algorithmic}
\usepackage{enumitem}
\usepackage[most]{tcolorbox}

\makeatletter
\@ifundefined{theorem}{\newtheorem{theorem}{Theorem}}{}
\@ifundefined{lemma}{\newtheorem{lemma}[theorem]{Lemma}}{}
\@ifundefined{corollary}{}{}
\@ifundefined{proposition}{}{}
\@ifundefined{remark}{\newtheorem{remark}{Remark}}{}
\@ifundefined{assumption}{\newtheorem{assumption}{Assumption}}{}
\makeatother

\title{POLAR: Online Learning for LoRA Adapter Caching and Routing in Edge LLM Serving}

\author{Shaoang Li}
\affiliation{
  \institution{Stony Brook University}
  \country{}}
\email{shaoang.li@stonybrook.edu}

\author{Jian Li}
\affiliation{
  \institution{Stony Brook University}
  \country{}}
\email{jian.li.3@stonybrook.edu}

\begin{document}

\begin{abstract}
Edge deployment of large language models (LLMs) increasingly relies on libraries of lightweight LoRA adapters, yet GPU/DRAM can keep only a small resident subset at a time. Serving a request through a non-resident adapter requires paging its weights from storage, incurring measurable latency. This creates a two-timescale online control problem: on a slow timescale, the system selects which adapters remain resident in fast memory, while on a fast timescale it routes each request to an adapter whose context-dependent utility is unknown a priori. The two decisions are tightly coupled: the cache determines the cost of exploration, and the router determines which adapters receive informative feedback. We formulate this joint caching-and-routing problem as a two-timescale contextual bandit and propose \textsc{POLAR} (Paging and Online Learning for Adapter Routing). POLAR pairs a cache-aware LinUCB router with an epoch-based cache controller. We study two variants. A fixed-epoch version provides a robust baseline with worst-case regret guarantees under arbitrary contexts. An epoch-doubling version, \textsc{POLAR+}, adds forced exploration and improved cache optimization to achieve $\widetilde{\mathcal{O}}(d\sqrt{NT}+\sqrt{KT})$ sublinear regret under stochastic regularity and cacheability conditions, where $N$ is the adapter count, $K$ the cache size, $d$ the context dimension, and $T$ the horizon. The routing term matches the standard contextual-bandit rate up to logarithmic factors, showing that the memory hierarchy does not fundamentally slow routing learning. Experiments using 15 real LoRA adapters for Qwen2.5-7B together with measured GPU paging latencies show that adaptive cache control substantially outperforms non-adaptive baselines and exhibits scaling trends consistent with the theory.
\end{abstract}

\maketitle

\section{Introduction}\label{sec:intro}

Edge deployment of large language models (LLMs) is increasingly moving toward a modular design: a single base model is kept on device, while many downstream functionalities are implemented through lightweight adapters such as LoRA modules~\cite{hu2021lora}. This architecture is attractive because it avoids replicating the full base model for every task or tenant, and makes it practical to support diverse behaviors including translation, summarization, recommendation, code generation, and domain-specific assistance on resource-constrained edge servers. However, this modularity introduces a systems bottleneck that is easy to overlook: while the adapter library may contain tens or hundreds of specialized adapters, GPU/DRAM can hold only a small subset at any time. The rest remain on local storage and can be accessed only through a slower path that incurs nontrivial loading delay~\cite{sheng2024slora,shen2025edgelora,li2024caraserve}.

\begin{figure*}[t]
  \centering
  \includegraphics[width=0.95\textwidth]{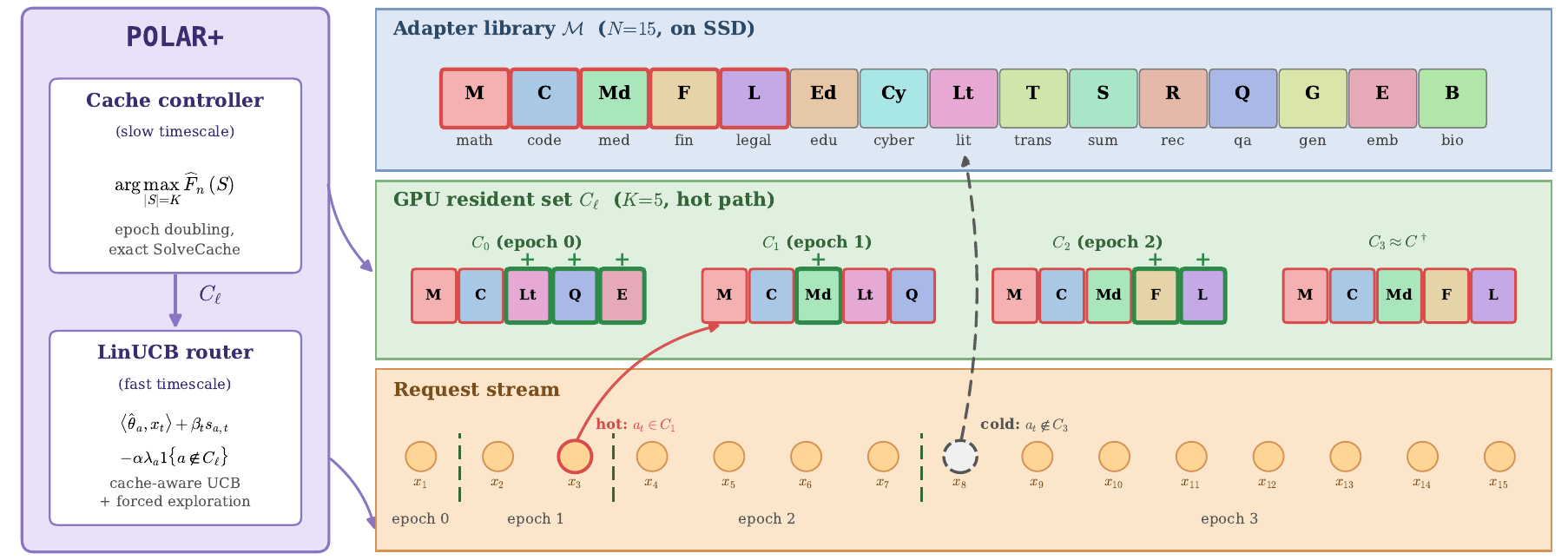}
  \vspace{-0.15in}
  \caption{Overview of \textsc{POLAR+}. The \emph{cache controller} (slow timescale) updates the resident set $C_\ell$ at each epoch boundary; the \emph{LinUCB router} (fast timescale) selects an adapter for each request $x_t$. The library $\mathcal{M}$ holds $N$ adapters on SSD; the GPU keeps a resident set of size $K$ that evolves $C_1 \to C_2 \to C_3 \to C_4 \approx C^\dagger$ as epoch lengths double (green ``$+$'' marks newly admitted adapters). Request $x_3$ hits a cached adapter (hot path); request $x_8$ needs an evicted adapter and pays the cold-path penalty.}
  \label{fig:arch}
  \vspace{-0.1in}
\end{figure*}

This memory-hierarchy constraint makes adapter residency a first-order systems decision rather than a mere implementation detail. In an interactive edge service, routing a request to a resident adapter yields a low-latency \emph{hot} execution path, whereas routing it to a non-resident adapter requires a \emph{cold} path, e.g., loading weights from SSD or flash into fast memory, which can substantially increase time-to-first-token and tail latency. Existing systems work has shown that LoRA loading and memory management are critical to end-to-end serving efficiency, but the missing online decision is \emph{which} adapters should remain resident in fast memory in the first place. That decision is inherently dynamic: the most valuable resident set depends on the request stream, and that value is not known in advance.

This yields a natural two-timescale architecture. On a slow timescale, the system control plane selects a resident working set of at most $K$ adapters to keep in GPU/DRAM. Frequent reconfiguration is undesirable because loading new adapters incurs overhead and can destabilize service, so the resident set should change only periodically. On a fast timescale, the data plane routes each incoming request to an adapter. Crucially, the request need not be restricted to the current resident set: a non-resident adapter may still be the best choice for the query, but selecting it incurs an additional latency penalty. The resulting problem is therefore neither standard caching nor standard routing. The cache determines which adapters are cheap to access, while routing determines which adapters are used and thus which ones generate informative feedback for future decisions.

The challenge is amplified by uncertainty. Adapter quality is request-dependent and typically unknown a priori: an adapter that performs well on one class of queries may be poor on another. Hence the system must learn, online, how adapter utility varies with context while simultaneously managing a scarce memory resource. This creates an endogenous feedback loop. If the resident set is poorly chosen, potentially useful adapters remain expensive to probe, slowing learning. If routing is shortsighted, the system may collect feedback only on currently cheap adapters and fail to discover better ones. In other words, the value of caching an adapter is not exogenous; it depends on the future routing policy, which itself depends on what has been learned so far.

This raises the central question of the paper:
\vspace{-0.1in}
\begin{tcolorbox}[colback=white!5!white,colframe=white!75!white]
\emph{How can an edge LLM service jointly learn which adapters to route requests to and which adapters to keep resident in fast memory, when adapter utility is context-dependent, cache updates operate on a slower control-plane timescale, and non-resident access incurs a measurable latency penalty?}
\end{tcolorbox}
\vspace{-0.15in}

This coupled control problem differs fundamentally from prior work in adjacent areas. If all adapters were equally accessible, it would reduce to contextual model selection. If adapter quality were known ahead of time, it would reduce to cache management with fixed values; and if the resident set could be changed on every request without cost, it would collapse to a one-timescale paging model. None of these abstractions captures practical edge LLM serving. Existing systems such as S-LoRA~\cite{sheng2024slora}, FastLibra~\cite{zhang2025fastlibra}, and EdgeLoRA~\cite{shen2025edgelora} optimize adapter loading and memory sharing, but do not study the online decision problem of jointly learning adapter utility and managing residency under request-dependent performance.

We study this problem through the lens of edge LLM serving and formulate it as a two-timescale online learning problem. The system hosts a fixed base LLM and a library of adapters on storage. Time is divided into epochs. At each epoch boundary, the controller selects a resident set of size at most $K$. During the epoch, each request arrives with a context vector, and the router selects an adapter from the full library. If the selected adapter is resident, the request follows the hot path; otherwise, it is served through the cold path and pays an additional latency penalty. We focus on the practically important regime in which the workload is approximately stable over the timescale of cache reconfiguration, so the target operating point is a steady resident set rather than continual cache churn. In this setting, the core online task is to learn which adapters deserve residency while still making good request-level routing decisions.

To address this challenge, we propose \textsc{POLAR} (\textbf{P}aging and \textbf{O}nline \textbf{L}earning for \textbf{A}dapter \textbf{R}outing), a framework that mirrors the system hierarchy (Figure~\ref{fig:arch}). \textsc{POLAR} combines a fast-timescale cache-aware LinUCB router with a slow-timescale cache controller that updates the resident set only at epoch boundaries. We study two variants. The first, \textsc{POLAR}, uses fixed-length epochs and greedy cache updates, and serves as a robust baseline that applies under arbitrary context sequences. The second, \textsc{POLAR+}, augments this design with forced exploration, exact cache optimization, and geometrically increasing epochs. The role of \textsc{POLAR+} is not merely to improve constants: it overcomes the persistent cache-identification difficulty that remains in the unrestricted setting, and under additional workload regularity and cacheability conditions it achieves sublinear regret.

The analysis reveals a technical obstacle that does not arise in either standard contextual bandits or standard online caching. In ordinary contextual bandits, uncertainty about an arm depends on how often that arm is sampled. Here, sampling itself is shaped by the memory hierarchy: resident adapters are cheaper to explore, while non-resident adapters are more expensive, so the cache changes the data-collection process. As a result, the slow-timescale cache decision is statistically coupled with the fast-timescale learner. Our results show that this coupling can nevertheless be controlled. For the fixed-epoch variant, we establish a worst-case regret guarantee under the linear contextual model, and show that the resulting linear cache gap is intrinsic under arbitrary contexts rather than an artifact of the proof. For \textsc{POLAR+}, we show that under stationary workload regularity and a natural cacheability condition, epoch doubling and forced exploration make the cumulative cache-identification error summable, yielding an overall sublinear regret bound. The routing term matches the standard contextual-bandit rate up to logarithmic factors, showing that the memory hierarchy does not fundamentally slow routing learning. To the best of our knowledge, this is the first such guarantee for joint adapter caching and contextual routing in edge LLM serving.

Our contributions are summarized as follows.
\begin{itemize}
\item We formulate edge LoRA serving as a two-timescale online control problem over a memory hierarchy, coupling slow resident-set selection with fast contextual routing under hot/cold latency asymmetry.
\item We propose \textsc{POLAR} and \textsc{POLAR+}, two two-timescale algorithms for joint adapter routing and cache placement. \textsc{POLAR} is a robust baseline under arbitrary contexts, while \textsc{POLAR+} adds epoch doubling, forced exploration, and improved cache optimization to obtain provably sublinear regret under additional structure.
\item We prove a worst-case regret bound for \textsc{POLAR} and show that its linear cache gap is intrinsic without workload regularity. We further prove a sublinear regret bound for \textsc{POLAR+} under stochastic regularity and cacheable diversity, with a routing term matching the standard contextual-bandit rate up to logarithmic factors.
\item We evaluate the proposed methods using real LoRA adapters on Qwen2.5-7B together with measured GPU paging latencies. The results show that adaptive cache control substantially reduces regret relative to non-adaptive baselines and exhibits scaling trends consistent with the theory.
\end{itemize}

\section{System Model}\label{sec:problem}

We consider an edge device hosting a fixed base LLM and a library of $N$ lightweight adapters (e.g., LoRA modules), indexed by $\mathcal{M}=\{1,\dots,N\}$. The full library resides on local storage, while only $1\le K\ll N$ adapters can be kept \emph{resident} in fast memory (GPU/DRAM). A non-resident adapter can still serve a request through a slower path, e.g., paging from storage into fast memory, incurring an additional latency penalty.

Time is slotted into rounds $t=1,\dots,T$, grouped into epochs $\ell=1,2,\dots$. At each round, the system observes a context $x_t\in\mathbb{R}^d$ with $\|x_t\|_2\le 1$. The problem operates on two timescales. On a \emph{slow} timescale, the resident set $C_\ell\subseteq\mathcal{M}$ with $|C_\ell|\le K$ is chosen at epoch boundaries and remains fixed during the epoch. On a \emph{fast} timescale, the system selects an adapter $a_t\in\mathcal{M}$ for each request, incurring low latency if $a_t\in C_{\ell(t)}$ (\emph{hot}), where $\ell(t)$ denotes the epoch containing round $t$, and an additional latency penalty otherwise (\emph{cold}). This two-timescale abstraction reflects practical edge serving stacks, where cache reconfiguration is a slower control-plane operation, while request routing is a fast data-plane decision.

\textbf{Reward model.} Each adapter $a$ has an unknown parameter $\theta_a^\star\in\mathbb{R}^d$ with $\|\theta_a^\star\|_2\le 1$, and a known cold-path penalty $\lambda_a>0$, which captures the additional latency of serving a request through a non-resident adapter. This penalty depends on the adapter and the memory/storage hierarchy, and is assumed measurable by profiling before deployment. The context $x_t$ encodes request features such as task type, input length, or domain, and $\theta_a^\star$ captures how adapter $a$ responds to these features. The observed quality is
\begin{align}
  q_t(a)=\langle \theta_a^\star,x_t\rangle+\eta_t(a),
  \label{eq:linmodel}
\end{align}
where $\eta_t(a)$ is conditionally $\sigma$-sub-Gaussian given $x_t$, and only $q_t(a_t)$ is observed. The per-round reward is
\begin{align}
  r_t(a_t;C_{\ell(t)})
  = q_t(a_t)-\alpha\lambda_{a_t}\mathbf{1}\{a_t\notin C_{\ell(t)}\},
  \label{eq:reward}
\end{align}
where $\alpha>0$ trades off service quality against latency. Thus, routing to a cold adapter is allowed, but pays an explicit penalty.

Reconfiguring the resident set between epochs also incurs overhead. We model this by a switching cost $\gamma |C_\ell\setminus C_{\ell-1}|$ with $\gamma>0$, which reflects the one-time cost of loading each newly admitted adapter into the cache at epoch boundaries. An online policy selects $\{a_t\}$ and $\{C_\ell\}$ to maximize
\begin{align}
  \max_{\pi}\mathbb{E}\!\left[\sum_{t=1}^T r_t(a_t;C_{\ell(t)})\!-\!\sum_{\ell\ge 2}\gamma |C_\ell\setminus C_{\ell-1}|\right]\quad \text{s.t.}~|C_\ell|\le K, \forall \ell.
\end{align}

\textbf{Benchmark and regret.} Define the noiseless reward
\begin{align}
  \mu_t(a;C):= \langle \theta_a^\star,x_t\rangle- \alpha\lambda_a\mathbf{1}\{a\notin C\}.
\end{align}
We benchmark against an oracle that knows $\{x_t\}_{t=1}^T$ and $\{\theta_a^\star\}_{a\in\mathcal{M}}$ in hindsight and selects a \emph{single fixed cache} $C_T^\star$ together with the optimal per-round routing rule under that cache. We benchmark against a fixed resident set for two reasons. First, frequent cache reconfiguration is costly and undesirable in practical edge systems, so the target operating point is typically a stable working set rather than continual cache churn. Second, in the approximately stationary regime, a fixed cache provides the appropriate steady-state benchmark: the online problem is not to track arbitrary drift, but to learn which adapters deserve residency while paying routing exploration and occasional reconfiguration costs.

The oracle value is
\begin{align}\label{eq:oracle}
  R^\star(T)= \max_{|C|=K}\sum_{t=1}^T \max_{a\in\mathcal{M}} \mu_t(a;C).
\end{align}
Accordingly, we define the pseudo-regret as the gap between the online policy and the best fixed-cache policy in hindsight, while charging the online policy for the cache reconfiguration overhead incurred during learning:
\begin{align}\label{eq:regret}
  \mathrm{Regret}(T)= R^\star(T)- \left[\sum_{t=1}^T \mu_t(a_t;C_{\ell(t)})- \sum_{\ell\ge 2}\gamma |C_\ell\setminus C_{\ell-1}|\right].
\end{align}
This regret depends only on the conditional mean rewards $\mu_t$, not on the noise realizations. Intuitively, it measures how far the online policy is from the best stationary cache-and-routing policy, plus the control-plane overhead required to discover that policy.

Under i.i.d. contexts, it is useful to define the \emph{population-optimal cache}
\begin{align}
  C^\dagger := \arg\max_{|C|=K} F(C),~\text{where}~ F(C) := \mathbb{E}_{x}\!\left[\max_{a\in\mathcal{M}} \mu(a;C;x)\right].
\end{align}
The hindsight benchmark $R^\star(T)/T$ is then the empirical counterpart of the population objective $F(C^\dagger)$. The gap between them is controlled by the uniform deviation bound in Lemma~\ref{lem:cacheid} in Appendix~\ref{app:polar_plus}, at rate $\mathcal{O}\big(\bar B\sqrt{K\log(eN/K)/T}\big)$, where $\bar B$ is a uniform per-round reward bound determined by $\alpha$ and $\lambda_{\max}$ (explicit form in Section~\ref{sec:analysis}).

\section{Algorithms}
\label{sec:algorithm}

We present two variants of \textbf{POLAR} (\textbf{P}aging and \textbf{O}nline \textbf{L}earning for \textbf{A}dapter \textbf{R}outing). Both follow the same two-timescale design. On the \emph{fast} timescale, a cache-aware LinUCB router selects an adapter for each request. On the \emph{slow} timescale, a cache controller updates the resident set only at epoch boundaries. The key difference between the two variants is how the resident set is updated and how exploration is enforced.

\subsection{POLAR: A Fixed-Epoch Baseline}

We begin with a simple and practical baseline, \textsc{POLAR} (Algorithm~\ref{alg:polar}), which uses fixed-length epochs of $H$ rounds. Within each epoch, the resident set remains fixed, and the router computes for every adapter a cache-aware upper-confidence score
\begin{align}
\mathrm{score}_t(a)= \langle \hat\theta_a,x_t\rangle +\beta_t\sqrt{x_t^\top V_a^{-1}x_t} -\alpha\lambda_a\mathbf{1}\{a\notin C_\ell\}.
\end{align}
The first two terms are the standard LinUCB estimate and uncertainty bonus, where $\hat\theta_a$ and $V_a$ are the ridge estimate and regularized design matrix maintained by the algorithm (explicit form in Section~\ref{sec:analysis}) and $\beta_t$ is the confidence radius specified in Lemma~\ref{lem:ucb}; the last term explicitly penalizes routing through the cold path. Thus, \textsc{POLAR} prefers resident adapters unless a non-resident adapter offers sufficient estimated quality gain to offset its latency penalty.

At the end of each epoch, POLAR updates the resident set using a greedy marginal-gain rule based on the previous epoch's contexts and current UCB estimates; the full pseudocode appears in Appendix~\ref{app:greedy}. Because this update incurs approximation loss and rarely selected adapters may remain poorly estimated, \textsc{POLAR} does not in general guarantee sublinear regret under our strongest benchmark (see Theorem~\ref{thm:wc}).

\begin{algorithm}[t]
\caption{\textsc{POLAR}: Fixed-Epoch Routing with Greedy Cache Updates}
\label{alg:polar}
\begin{algorithmic}[1]
\STATE \textbf{Input:} epoch length $H$, cache size $K$, ridge $\lambda>0$,
  confidence sequence $\{\beta_t\}$, latency weight $\alpha$, switching cost $\gamma$
\STATE \textbf{Init:} $V_a\leftarrow \lambda I_d$, $b_a\leftarrow 0$,
  $\hat\theta_a\leftarrow 0$ for all $a\in\mathcal{M}$;
  choose initial cache $C_1$ with $|C_1|\le K$
\FOR{epoch $\ell=1,2,\dots,\lceil T/H\rceil$}
  \IF{$\ell>1$}
    \STATE $C_\ell \leftarrow
      \textsc{GreedyCacheUpdate}\!\left(
      C_{\ell-1},\{\hat\theta_a,V_a\},
      \{x_t\}_{t\in \mathrm{epoch}\,\ell-1}\right)$
    \STATE Pay switching cost $\gamma |C_\ell\setminus C_{\ell-1}|$
  \ENDIF
  \FOR{$t=(\ell-1)H+1,\dots,\min\{\ell H,T\}$}
    \STATE Observe context $x_t$
    \FOR{each $a\in\mathcal{M}$}
      \STATE
      $\mathrm{score}_t(a)\leftarrow
      \langle \hat\theta_a,x_t\rangle
      + \beta_t\sqrt{x_t^\top V_a^{-1}x_t}
      - \alpha\lambda_a\mathbf{1}\{a\notin C_\ell\}$
    \ENDFOR
    \STATE Play $a_t\in\arg\max_{a}\mathrm{score}_t(a)$ and observe $q_t(a_t)$
    \STATE Update $V_{a_t}\leftarrow V_{a_t}+x_tx_t^\top$,
      $b_{a_t}\leftarrow b_{a_t}+q_t(a_t)x_t$,
      $\hat\theta_{a_t}\leftarrow V_{a_t}^{-1}b_{a_t}$
  \ENDFOR
\ENDFOR
\end{algorithmic}
\end{algorithm}

\subsection{POLAR+: Epoch Doubling with Forced Exploration}

While \textsc{POLAR} captures the basic two-timescale control logic, its regret can remain linear due to two effects. First, the greedy cache update introduces a constant-factor approximation loss at every epoch. Second, adapters that are rarely selected by the router may never accumulate sufficiently accurate estimates, which in turn prevents the cache controller from confidently identifying the best resident set. To address both issues, we introduce \textsc{POLAR+} (Algorithm~\ref{alg:polar_plus}), which is the main algorithmic variant studied in this paper. \textsc{POLAR+} strengthens \textsc{POLAR} in three ways.

\textbf{(i) Epoch doubling.} Epoch $\ell$ contains $2^\ell$ exploitation rounds. As the epoch index grows, parameter estimates become more accurate, while the number of cache updates grows only logarithmically in $T$. This reduces control-plane churn and is key to making the cumulative cache-identification error summable.

\textbf{(ii) Forced exploration.} At the start of each epoch, \textsc{POLAR+} executes a short forced exploration phase consisting of $Nm_\ell$ round-robin plays, where each adapter is sampled $m_\ell$ times independent of the current router estimates. This ensures every adapter receives informative observations, including adapters that would otherwise remain under-sampled because their current estimates or cold penalties make them unattractive. Operationally, this phase can be interpreted as a probing window inserted between cache updates.

\begin{algorithm}[t]
\caption{\textsc{POLAR+}: Epoch Doubling with Forced Exploration}
\label{alg:polar_plus}
\begin{algorithmic}[1]
\STATE \textbf{Input:} cache size $K$, ridge $\lambda>0$,
  confidence sequence $\{\beta_t\}$, latency weight $\alpha$,
  switching cost $\gamma$, failure level $\delta$, exploration constant $\kappa>0$.
\STATE \textbf{Init:} $V_a\leftarrow \lambda I_d$, $b_a\leftarrow 0$,
  $\hat\theta_a\leftarrow 0$ for all $a\in\mathcal{M}$;
  $c_0\leftarrow \lceil \log(6Nd/\delta)\rceil$
\FOR{epoch $\ell=0,1,2,\dots$ while $t\le T$}
  \STATE $m_\ell \leftarrow \kappa \cdot d(\ell+c_0)$
  \STATE \textbf{Forced exploration:}
  for the next $Nm_\ell$ rounds,
  observe $x_t$, play the pre-scheduled round-robin adapter $a_t$,
  observe $q_t(a_t)$, and update
  $V_{a_t}, b_{a_t}, \hat\theta_{a_t}$
  \STATE $C_\ell \leftarrow
  \textsc{SolveCache}(\hat F_{n_\ell+Nm_\ell},K)$
  \STATE Pay switching cost $\gamma |C_\ell\setminus C_{\ell-1}|$
  if $\ell>0$
  \FOR{the next $2^\ell$ rounds (or until $T$ is exhausted)}
    \STATE Observe context $x_t$
    \FOR{each $a\in\mathcal{M}$}
      \STATE
      $\mathrm{score}_t(a)\leftarrow
      \langle \hat\theta_a,x_t\rangle
      + \beta_t\sqrt{x_t^\top V_a^{-1}x_t}
      - \alpha\lambda_a\mathbf{1}\{a\notin C_\ell\}$
    \ENDFOR
    \STATE Play $a_t\in\arg\max_a \mathrm{score}_t(a)$ and observe $q_t(a_t)$
    \STATE Update $V_{a_t}\leftarrow V_{a_t}+x_tx_t^\top$,
      $b_{a_t}\leftarrow b_{a_t}+q_t(a_t)x_t$,
      $\hat\theta_{a_t}\leftarrow V_{a_t}^{-1}b_{a_t}$
  \ENDFOR
\ENDFOR
\end{algorithmic}
\end{algorithm}

\textbf{(iii) Improved cache optimization.} After the forced exploration phase, the cache controller selects the next resident set by maximizing an empirical cache utility $\hat F_n(S)$ over all size-$K$ sets. For the theory, we write this step abstractly as $\textsc{SolveCache}(\hat F_n,K)$. Formally,
\begin{align}
\hat F_n(S)&:=\frac{1}{n}\sum_{t=1}^{n}\max_{a\in\mathcal{M}} \hat\mu(a;S;x_t),\nonumber\\
\hat\mu(a;S;x) &:= \langle \hat\theta_a,x\rangle
-\alpha\lambda_a\mathbf{1}\{a\notin S\},
\end{align}
is the empirical cache utility evaluated using the current parameter estimates over all historical contexts; $n_\ell := \sum_{j<\ell}(Nm_j+2^j)$ denotes the number of rounds completed before epoch $\ell$. In the analysis of \textsc{POLAR+}, we instantiate \textsc{SolveCache} as the exact optimizer $\arg\max_{|S|=K}\hat F_n(S)$ in order to isolate the statistical difficulty of cache identification from approximation effects. In practice, this step can be implemented by the same greedy procedure as Algorithm~\ref{alg:greedy_cache} or by any standard fast solver for monotone submodular maximization. When the adapter library is moderate, exact optimization is also feasible at epoch boundaries. Under i.i.d. contexts and the structural cacheability conditions in Assumptions~\ref{ass:iid}--\ref{ass:diversity}, \textsc{POLAR+} achieves $\widetilde{\mathcal{O}}(\sqrt{T})$ sublinear regret; see Theorem~\ref{thm:ed}.

Finally, note that neither \textsc{POLAR} nor \textsc{POLAR+} needs to know the latent hot set $\mathcal{M}_{\mathrm{hot}}$ from Assumption~\ref{ass:diversity}. Both routing and caching are carried out over the full adapter library $\mathcal{M}$. Moreover, \textsc{POLAR+} is \emph{anytime}: all quantities are computed online from the current epoch index and past observations, without requiring prior knowledge of the horizon $T$.

\section{Regret Analysis}
\label{sec:analysis}

We now analyze the two algorithmic variants introduced in Section~\ref{sec:algorithm}. This section establishes regret guarantees for the joint caching-and-routing problem and highlights the technical mechanisms that make these guarantees possible. In particular, we explain why the problem is not a routine combination of contextual bandits and online caching, and how the two-timescale structure shapes both the upper bounds and their proof techniques.

At a high level, the main difficulty comes from the two-timescale coupling. The fast router must learn context-dependent adapter quality from bandit feedback, while the slow cache controller determines which adapters are cheap to access. This creates an endogenous feedback loop: the resident set changes the cost of exploration, and the routing policy determines which adapters receive informative observations. Consequently, caching is not merely a combinatorial action layered on top of a contextual bandit; it alters the data-collection process that drives learning.

We first analyze \textsc{POLAR} under arbitrary contexts and show that it achieves a worst-case regret guarantee, but with a linear cache gap. We then show that this linear term is not a proof artifact of the analysis of \textsc{POLAR}. Rather, it reflects the intrinsic difficulty of cache identification without additional workload structure. Next, we study \textsc{POLAR+} under stationary workload regularity and cacheability assumptions, and show that under this added structure the cache-identification cost becomes summable, yielding an overall sublinear regret bound.

Throughout the analysis, we use the following notation. Let $ V_{a,t} := \lambda I_d + \sum_{s<t}\mathbf{1}\{a_s=a\}x_sx_s^\top $ be the regularized design matrix of adapter $a$ up to round $t$, and $ \hat\theta_{a,t} := V_{a,t}^{-1} \sum_{s<t}\mathbf{1}\{a_s=a\}q_s(a)x_s $ its ridge estimate. Let $ s_{a,t}:=\sqrt{x_t^\top V_{a,t}^{-1}x_t} $ denote the corresponding prediction width. We write $ \bar B := 2+\alpha\lambda_{\max}, \lambda_{\max}:=\max_{a\in\mathcal{M}}\lambda_a, $ so that $\bar B$ is a uniform upper bound on the per-round pseudo-regret, since the quality difference is at most $2$ and the largest cold-path penalty is $\alpha\lambda_{\max}$.

\subsection{Worst-Case Guarantee for \textsc{POLAR}}
\label{subsec:wc}

We begin with the fixed-epoch baseline \textsc{POLAR}. This result applies under the basic linear contextual model with no distributional assumptions on the context sequence. It should therefore be viewed as a robustness baseline: even under arbitrary request sequences, \textsc{POLAR} controls the routing error and cache switching overhead. At the same time, Theorem~\ref{thm:wc} makes clear that without additional workload structure, the cache-identification cost can remain linear.

\begin{theorem}[Worst-case regret of \textsc{POLAR}]
\label{thm:wc}
Assume the linear model~\eqref{eq:linmodel}, with $\|\theta_a^\star\|_2\le 1$, $\|x_t\|_2\le 1$, conditionally $\sigma$-sub-Gaussian noise, ridge parameter $\lambda\ge 1$, and bounded cold-path penalties $\lambda_a\in(0,\lambda_{\max}]$. Then, with probability at least $1-\delta$, Algorithm~\ref{alg:polar} satisfies
\begin{align}
\mathrm{Regret}(T)\le \underbrace{2\beta_T\sqrt{2NdT\log\Bigl(1+\frac{T}{d\lambda}\Bigr)}}_{\text{routing}} +\underbrace{\alpha\lambda_{\max}T}_{\text{cache gap}} +\underbrace{\Bigl\lceil \frac{T}{H}\Bigr\rceil K\gamma}_{\text{switching}}.
\end{align}
\end{theorem}

\textbf{Implications.} Theorem~\ref{thm:wc} separates the overall regret into three distinct components.

$\triangleright$ \textbf{Routing regret}: The first term is the standard statistical price of learning context-dependent adapter quality under bandit feedback. Up to constants and logarithmic factors, this is the familiar LinUCB term.

$\triangleright$ \textbf{Cache gap}: The second term measures the cost of operating with a possibly suboptimal resident set. This is the extra difficulty introduced by the memory hierarchy: even if the router were optimal for the current cache, the cache itself may still be wrong.

$\triangleright$  \textbf{Switching overhead}: The third term is the control-plane cost of reconfiguring the resident set at epoch boundaries. This term has no analog in ordinary contextual bandits which do not model slow-timescale resident-set updates.

Taken together, Theorem~\ref{thm:wc} shows that regret in edge adapter serving has three sources: a data-plane learning term, a cache-identification term induced by the memory hierarchy, and a control-plane switching term caused by cache reconfiguration.

\begin{remark}[Why standard analyses do not apply directly]
\label{rem:wc-nonstandard}
Classical contextual bandit analyses control routing regret for a fixed action set, but do not model a slow-timescale resident-set decision that changes the effective access cost of each arm; see, e.g., standard analyses of linear contextual bandits and LinUCB \cite{chu2011contextual,abbasi2011improved}. Conversely, online caching and paging analyses typically assume that item values are known or exogenous, rather than learned from policy-dependent feedback. Our setting lies between these two worlds: the cache changes the future cost of observing adapters, and the router determines which adapters generate informative samples for future cache updates. This feedback coupling is precisely what makes the analysis nonstandard.
\end{remark}

\begin{remark}[Why the linear cache term is intrinsic under arbitrary contexts]\label{rem:linear-intrinsic}
The linear cache term in Theorem~\ref{thm:wc} should not be interpreted as a weakness of \textsc{POLAR} alone. Rather, it reflects a genuine limitation of the unrestricted setting. Under arbitrary contexts, the long-run value of caching an adapter cannot, in general, be inferred from online observations alone, because the observations are themselves shaped by the current cache and routing policy. Without additional structure, the learner may never receive enough representative information to distinguish a globally valuable adapter from one that only appears useful under the current, possibly misleading, resident set. In this sense, the linear cache term captures the intrinsic difficulty of cache identification without workload regularity.
\end{remark}

\textbf{Practical guidance.} Theorem~\ref{thm:wc} also yields a simple systems-level tradeoff. A smaller epoch length $H$, equivalently a more frequent cache-update interval, allows the controller to react more quickly to newly learned information, but increases the switching term. A larger $H$ reduces reconfiguration overhead, but leaves the system with a stale resident set for longer. Moreover, when the cold-path penalties $\lambda_a$ are large, the cost of an incorrect resident set can dominate the routing regret, making cache identification the primary systems challenge. Thus, even the worst-case bound highlights why cache control and routing must be analyzed jointly rather than as separate components.

\textbf{Proof sketch.} The proof decomposes regret by adding and subtracting $\max_a \mu_t(a;C_{\ell(t)})$, which yields three terms corresponding to routing, cache mismatch, and switching. The routing term is handled by standard confidence and elliptic-potential arguments for LinUCB. The key simplification in the worst-case analysis is that the cache term is controlled by a uniform per-round bound: the gap between any two resident sets is at most $\alpha\lambda_{\max}$, so one need not estimate long-run cache utility under arbitrary contexts. The switching term simply counts epoch boundaries and charges at most $K\gamma$ per update. The full proof, together with the supporting confidence and submodularity lemmas, appears in Appendix~\ref{app:polar}.

\subsection{Sublinear Regret via Epoch Doubling and Cacheability (\textsc{POLAR+})}
\label{subsec:sublinear}

Theorem~\ref{thm:wc} shows that worst-case robustness alone is not enough to eliminate the linear cache gap. We now show that this gap disappears once the workload has stable statistical structure and the adapter library contains a meaningful cacheable subset. This is the regime most relevant to stationary edge services: request types recur over time, some adapters repeatedly serve distinct traffic segments well, and cold-path access is expensive enough that keeping the right adapters resident matters.

Theorem~\ref{thm:ed} shows that under this additional structure, \textsc{POLAR+} learns both \emph{which adapter to route to} and \emph{which adapters deserve residency} at sublinear cost. For analytical convenience we set $\kappa = 1$ throughout this section.

\begin{theorem}[Sublinear regret of \textsc{POLAR+}]\label{thm:ed}
Under Assumptions~\ref{ass:iid} and~\ref{ass:diversity}, and the conditions of Theorem~\ref{thm:wc}, Algorithm~\ref{alg:polar_plus} satisfies, with probability at least $1-\delta$,
\begin{align}
\mathrm{Regret}(T)=\widetilde{\mathcal{O}}\bigg(\underbrace{\sigma d\sqrt{NT}}_{\text{routing}}+\underbrace{\bar B\sqrt{KT}}_{\text{cache identification}}\bigg).
\end{align}
\end{theorem}

\textbf{Implications.} The routing term $\sigma d\sqrt{NT}$ matches the minimax lower bound $\Omega(d\sqrt{NT})$ for $N$-arm linear contextual bandits up to logarithmic factors and the noise level $\sigma$. This shows that the memory hierarchy does not fundamentally degrade routing efficiency: \textsc{POLAR+} learns adapter quality at the same rate as standard LinUCB, despite the additional caching layer. The cache-identification term $\bar B\sqrt{KT}$ is the extra price of learning which $K$ adapters deserve fast-memory residency. To the best of our knowledge, this is the first sublinear cache-identification guarantee for joint adapter caching and routing with unknown contextual utility.

\textbf{Regime interpretation.} Theorem~\ref{thm:ed} is especially informative across different operating regimes. When $K\ll N$ but the cold-path penalties are moderate, the routing term may dominate, and the main challenge is learning adapter quality. When cold-path access is expensive, the cache-identification term is the bottleneck, and learning the hot set is as important as learning per-request routing. In the degenerate case $K=N$, all adapters fit in memory and the cache term vanishes, reducing the problem to ordinary contextual routing. This matches the system intuition: when residency is no longer a constraint, the learning problem collapses to routing alone.

\textbf{What changes relative to Theorem~\ref{thm:wc}.} Compared with Theorem~\ref{thm:wc}, the linear cache gap becomes $\widetilde{\mathcal{O}}(\sqrt{T})$, and the switching overhead drops from $\mathcal{O}(T/H)$ to $\mathcal{O}(\log T)$ under epoch doubling. This improvement requires additional workload regularity and cacheability structure. Moreover, \textsc{POLAR+} is anytime and does not require prior knowledge of the horizon.

\textbf{Assumptions and operational meaning.} The sublinear result relies on two structural conditions that have clear interpretations in edge LLM serving. The first concerns workload stability. The second concerns whether the adapter library contains a meaningful hot set that is worth keeping resident.

\begin{assumption}[i.i.d. regularity]\label{ass:iid}
The contexts satisfy $x_t \stackrel{\text{i.i.d.}}{\sim}\mathcal{D}$ with
\begin{align}
\mathbb{E}[x_tx_t^\top] \succeq \sigma_{\min} I_d
\end{align}
for some $\sigma_{\min}>0$.
\end{assumption}

Assumption~\ref{ass:iid} formalizes a stationary workload over the cache-control timescale: the service repeatedly sees a stable mixture of request types, and the feature map contains enough variation to identify adapter quality. This assumption should be interpreted over the timescale of cache learning and reconfiguration, not as a claim that edge traffic is globally static over arbitrarily long horizons. Rather, it states that over the horizon on which the cache is being learned, the workload is stable enough to admit a meaningful steady-state resident set.

\begin{assumption}[Cacheable diversity]\label{ass:diversity}
There exists a subset $\mathcal{M}_{\mathrm{hot}}\subseteq\mathcal{M}$ with $|\mathcal{M}_{\mathrm{hot}}|\ge K$ such that for each $a\in\mathcal{M}_{\mathrm{hot}}$, the region
\begin{align}
X_a\!:=\!\left\{x\!:\!a\!=\!\arg\max_{a'\in\mathcal{M}}\mu(a';C;x)\text{ for all } C\subseteq \mathcal{M}_{\mathrm{hot}},\ |C|=K\right\}
\end{align}
satisfies:
\begin{enumerate}
\item[(i)] \textbf{Coverage:} $\Pr_{x\sim\mathcal{D}}(x\in X_a)\ge p_{\min}$.
\item[(ii)] \textbf{Margin:} for every $C\subseteq\mathcal{M}_{\mathrm{hot}}$ with $|C|=K$,
  \begin{align}
      \inf_{x\in X_a}\left[\mu(a;C;x)-\max_{a'\neq a}\mu(a';C;x)\right]\ge \Delta_{\min}>0.
    \end{align}
\item[(iii)] \textbf{Full-dimensionality:} $\mathbb{E}[xx^\top\mid x\in X_a]\succeq \sigma_{\min}I_d$.
\item[(iv)] \textbf{Cache separability:}
  \begin{align}
      F(C^\dagger)- \max_{\substack{|S|\le K\\ S\not\subseteq \mathcal{M}_{\mathrm{hot}}}} F(S) \ge \Delta_{\mathrm{cache}}>0.
   \end{align}
\item[(v)] \textbf{Hot dominance:}
 \begin{align}
      \alpha \cdot \min_{a\notin \mathcal{M}_{\mathrm{hot}}}\lambda_a > 2,
\end{align}
and we write
    \begin{align}
      \Delta_{\mathrm{dom}}:=\alpha \min_{a\notin \mathcal{M}_{\mathrm{hot}}}\lambda_a - 2 > 0.
    \end{align}
\end{enumerate}
\end{assumption}

Assumption~\ref{ass:diversity} captures the practical notion of a \emph{cacheable hot set}. In many edge services, traffic is skewed: a subset of adapters repeatedly serves recurring request segments, and those adapters are precisely the ones that should remain resident. The coverage and margin conditions say that each hot adapter is useful on a nontrivial part of the workload and is clearly better there. The cache-separability condition says that the optimal resident set is meaningfully better than any cache that includes non-hot adapters. Finally, the hot-dominance condition says that for non-hot adapters, the cold-path latency is large enough that any quality gain they might offer is not sufficient to offset the paging penalty. Operationally, this means the service has a meaningful steady-state working set and that cold execution is expensive enough to make residency matter. These assumptions should therefore be viewed not merely as technical devices, but as identifiability conditions for steady-state cache learning in a stationary edge service.

\begin{remark}[Why Assumption~\ref{ass:diversity} is not overly restrictive]\label{rem:assumption-practical}
Assumption~\ref{ass:diversity} requires diversity only for the cacheable hot subset, not for the entire adapter library. After the bootstrap phase, the algorithm's cache is shown to lie inside $\mathcal{M}_{\mathrm{hot}}$, so competition effectively occurs only among hot adapters. In particular, hot dominance ensures that cold adapters are never competitive under any hot cache once the estimation error is sufficiently small. This is the analog, in our two-timescale setting, of a margin condition in standard contextual bandits: it guarantees that the steady-state resident set is identifiable from online observations rather than requiring all adapters in the library to be equally learnable or equally relevant.
\end{remark}

\textbf{Practical guidance.} Theorem~\ref{thm:ed} suggests three design principles for edge LLM serving: cache updates should become less frequent as confidence improves; explicit probing is most valuable early, when the hot set is still being identified; and larger cold-path penalties increase the value of accurate cache selection. If memory is abundant enough to hold all adapters, the problem reduces to ordinary contextual routing.

\subsubsection{Proof Structure and Technical Insights}
${}$
\vspace{0.1cm}

\textbf{A more detailed decomposition.} The proof yields the following refined expansion:
\begin{align}\label{eq:refined-bound}
&\mathrm{Regret}(T)\!\leq\!\underbrace{2\beta_T\sqrt{2NdT\log\Bigl(1+\frac{T}{d\lambda}\Bigr)}}_{\text{routing}}
+ \underbrace{c_2\frac{\sigma\sqrt{d\log(NLT/\delta)}+\sqrt{\lambda}}{\sqrt{\sigma_{\min}p_{\min}}}\sqrt{T}}_{\text{cache estimation}} \nonumber\allowdisplaybreaks\\
&+ \underbrace{c_3\bar B\sqrt{KT\log\Bigl(\frac{eNL}{K\delta}\Bigr)}}_{\text{cache concentration}} + \underbrace{K\gamma L}_{\text{switching}} \!+\! \underbrace{\sum_{\ell}Nm_\ell \bar B}_{\text{forced exploration}} \!+\! \underbrace{2^{\ell_\star}\bar B}_{\text{bootstrap}},
\end{align}
where $L=\lceil \log_2 T\rceil$ and $\ell_\star$ is the bootstrap epoch defined later in the proof.

Each term has a distinct role. The first is the usual routing regret. The second is the price of estimating adapter utilities accurately enough for cache selection. The third is a uniform concentration term over all size-$K$ resident sets. The fourth is the total control-plane churn from epoch updates. The fifth is the explicit probing cost incurred by forced exploration. The sixth is a transient bootstrap term until the algorithm has collected enough information to identify the relevant hot set. The key point is that the first three terms are $\widetilde{\mathcal{O}}(\sqrt{T})$, while the last three are lower-order for fixed problem parameters.

\begin{remark}[Why standard proof techniques still do not apply]\label{rem:sublinear-nonstandard}
Even under Assumptions~\ref{ass:iid} and~\ref{ass:diversity}, the analysis is not a routine combination of contextual-bandit and caching arguments. The reason is that the resident set changes the information geometry of the problem. A non-resident adapter is not only more expensive to use; it is also less likely to be sampled by the optimistic router. Thus the algorithm may fail to learn about an adapter precisely because it is not resident, and may fail to cache it because it has not yet learned its value. This circular dependence does not appear in standard contextual bandits, nor in online caching with exogenous values.
\end{remark}

\textbf{Proof idea and technical novelties.} The proof has three conceptual stages.

$\triangleright$ \textbf{Stage 1: Bootstrap.} Forced exploration guarantees that every adapter receives i.i.d. samples regardless of the current cache. After a finite bootstrap epoch $\ell_\star$, two conditions hold simultaneously: (i) the routing confidence width is small enough to identify the best adapter on the relevant regions of the context space, and (ii) the empirical cache utility is accurate enough to distinguish hot caches from non-hot ones. This breaks the initial adverse feedback loop between cache and routing.

$\triangleright$ \textbf{Stage 2: Geometric convergence.} Once the bootstrap conditions hold, a virtuous cycle emerges. Better estimates improve routing and better routing yields more informative exploitation samples on hot regions. These samples sharpen the empirical cache utility, and better cache choices reduce cold-path penalties. Under epoch doubling, epoch $\ell$ has $2^\ell$ exploitation rounds while the residual per-epoch cache error decays as $\mathcal{O}(2^{-\ell/2})$. This is the key mechanism that turns a persistent cache gap into a summable one.

$\triangleright$ \textbf{Stage 3: Collecting terms.} The final regret bound follows by summing the routing term from the elliptic potential argument, the cache-estimation and concentration terms from the empirical cache utility, and the lower-order switching, probing, and bootstrap costs. The complete derivation appears in Appendix~\ref{app:polar_plus}.

\section{Experiments}
\label{sec:experiments}

We evaluate \textsc{POLAR} and \textsc{POLAR+} in a measurement-calibrated simulation using real LoRA adapters, measured GPU paging latencies, and benchmark-derived adapter utilities. The experiments are designed to validate three aspects of the theory: (i) the sublinear scaling of \textsc{POLAR+}, (ii) the role of epoch doubling, forced exploration, and improved cache optimization, and (iii) the ability of \textsc{POLAR+} to learn a high-quality resident set while keeping control-plane churn low.

\subsection{Experimental Setup}
\label{subsec:setup}

\textbf{Adapters and base model.} We use 15 community-contributed LoRA adapters for Qwen2.5-7B-Instruct~\cite{qwen2.5}, downloaded from HuggingFace and spanning 15 domains, including math, reasoning, medical, clinical NER, finance, legal, education, cybersecurity, and creative writing. Together with the base model, this gives $N=16$ arms. We include the base model as a special arm with $\lambda_{\text{base}}=0$, reflecting that it is always resident as part of the serving stack and therefore incurs no cold-path penalty; the cache controller thus effectively allocates its $K$ slots among the $15$ LoRA adapters. Adapters exhibit substantial heterogeneity in LoRA rank (8 to 64) and size (50 to 656 MB), which makes them suitable for studying memory-residency decisions under realistic edge constraints.

\textbf{Latency and quality calibration.} For each adapter $a$, we measure the cold-path latency $\lambda_a$ on an NVIDIA RTX~5080 GPU by loading the adapter through PEFT 20 times and recording the loading delay. The measured GPU load latencies range from 291 ms (rank-32, 50 MB) to 1263 ms (rank-64, 656 MB), confirming substantial heterogeneity in paging cost across adapters. In the simulation, $\lambda_a$ is expressed in seconds ($\lambda_a=\text{raw ms}/10^3$), matching the reward scale of the quality parameters $\theta_a^\star$ (normalized so that $\|\theta_a^\star\|_2\le 1$). To instantiate the quality model, each adapter is evaluated on five benchmarks (MMLU, ARC-Challenge, HellaSwag, TruthfulQA, and GSM8K) using log-probability scoring, yielding the adapter-specific utility parameters $\theta_a^\star$ used in the simulation. In particular, the measured adapter utilities differ meaningfully across domains: for example, the job-classifier adapter reduces MMLU accuracy from 0.70 to 0.42, while the cybersecurity adapter improves it to 0.74.

\textbf{Request generation and simulation.} Contexts follow a power-law task distribution as in EdgeLoRA~\cite{shen2025edgelora}. Thus, the evaluation should be understood as \emph{measurement-calibrated} rather than a literal trace replay: real adapter latencies and benchmark-derived utilities are combined with a synthetic but workload-informed request stream. Each request is associated with a context vector encoding task information, and rewards are generated according to the linear contextual model in Section~\ref{sec:problem}. This design allows us to study the online control problem under realistic paging costs while keeping the underlying utility model transparent and reproducible.

\textbf{Simulation parameters.} Unless otherwise stated, we use cache size $K=5$, context dimension $d=5$, epoch length $H=200$, noise level $\sigma=0.05$, switching cost $\gamma=0.3$, and ridge parameter $\lambda=1$. We set $\alpha=0.5$ so that the average cache-miss penalty is comparable to the average utility gap (penalty/quality ratio $\approx 1.1$). For \textsc{POLAR+}, forced exploration uses $\kappa=0.05$, corresponding to $Nm_\ell\approx 32$ rounds per epoch, about $2\%$ of the horizon. All reported results are means $\pm$ standard deviations over 5 random seeds.

\textbf{Baselines.} We compare against four online baselines that pair LinUCB routing with standard cache policies: \emph{$\varepsilon$-Greedy Cache} ($\varepsilon=0.1$ with greedy submodular cache updates), \emph{LFU Cache}, \emph{LRU Cache}, and \emph{Static Cache} (a random fixed cache with no updates). We additionally include an \emph{Oracle Cache + LinUCB} reference, in which the cache is fixed at the hindsight-optimal set $C_T^\star$ (defined in Section~\ref{sec:problem}) and LinUCB learns only the routing policy. This reference isolates the routing-side difficulty that remains even when cache learning is removed, and therefore serves as a routing-only lower reference for any online method that must learn the cache as well.

\begin{figure}[t]
  \centering
  \begin{minipage}[t]{0.48\columnwidth}
    \centering
    \includegraphics[width=\textwidth]{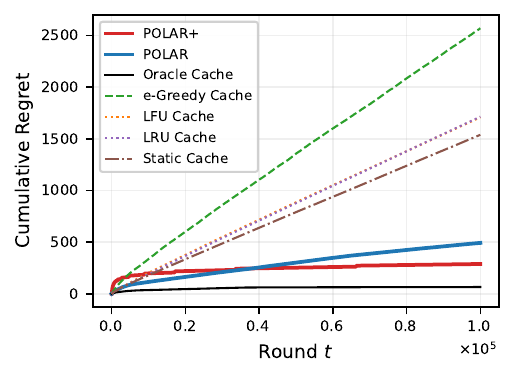}
    \Description{Cumulative regret trajectory at T=100K.}
    \vspace{-0.3in}
    \caption{Cumulative regret. \textsc{POLAR+} flattens; heuristic baselines grow linearly.}
    \label{fig:main_regret}
  \end{minipage}\hfill
  \begin{minipage}[t]{0.48\columnwidth}
    \centering
    \includegraphics[width=\textwidth]{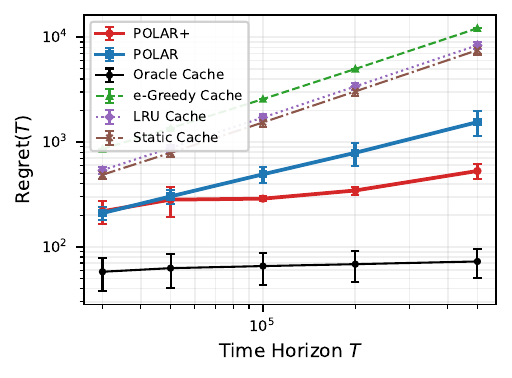}
    \Description{Log-log regret scaling to T=500K.}
    \vspace{-0.3in}
    \caption{Regret vs. $T$ (log-log). \textsc{POLAR+} overtakes \textsc{POLAR} at $T\approx 5\times 10^4$.}
    \label{fig:main_scaling}
  \end{minipage}
  \vspace{-0.2in}
\end{figure}

\subsection{Main Results}
\label{subsec:main_results}

Figure~\ref{fig:main_regret} reports cumulative regret up to $T=10^5$. Both \textsc{POLAR} and \textsc{POLAR+} significantly outperform the heuristic baselines, with \textsc{POLAR+} achieving the lowest regret among all online methods. Relative to the Oracle Cache reference, \textsc{POLAR+} incurs only a modest overhead for learning the resident set online, while substantially closing the gap between heuristic caching and the offline-optimal cache. At $T=10^5$, \textsc{POLAR+} reaches $288\pm 33$, compared with $493\pm 198$ for \textsc{POLAR}$,$ $1{,}538$ for the best heuristic baseline, and $66\pm 50$ for Oracle Cache. Its regret curve visibly flattens after $T\approx 3\times 10^4$, consistent with sublinear growth.

Figure~\ref{fig:main_scaling} shows regret as a function of $T$ on a log-log scale. \textsc{POLAR+} grows much more slowly than \textsc{POLAR} and the heuristic baselines, and its slope tracks Oracle Cache much more closely. In particular, \textsc{POLAR+} grows from $219$ to $530$ over a $16.7\times$ increase in the horizon, while \textsc{POLAR} and the heuristic baselines grow substantially faster. \textsc{POLAR+} overtakes \textsc{POLAR} near $T=5\times 10^4$, and the gap widens as the horizon grows. This behavior empirically supports the $\widetilde{\mathcal{O}}(\sqrt{T})$ scaling predicted by Theorem~\ref{thm:ed}.

\begin{figure}[t]
  \centering
  \begin{minipage}[t]{0.48\columnwidth}
    \centering
    \includegraphics[width=\textwidth]{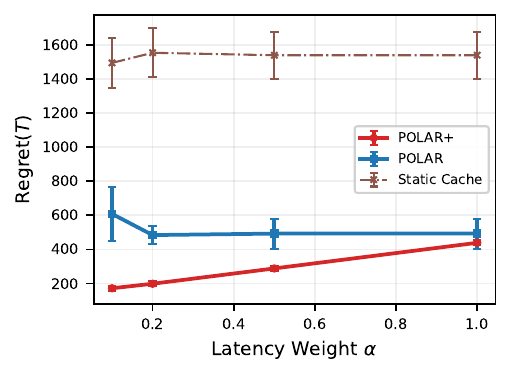}
    \Description{Regret vs latency weight alpha.}
       \vspace{-0.3in}
    \caption{Regret vs. $\alpha$.}
    \label{fig:sens_alpha}
  \end{minipage}\hfill
  \begin{minipage}[t]{0.48\columnwidth}
    \centering
    \includegraphics[width=\textwidth]{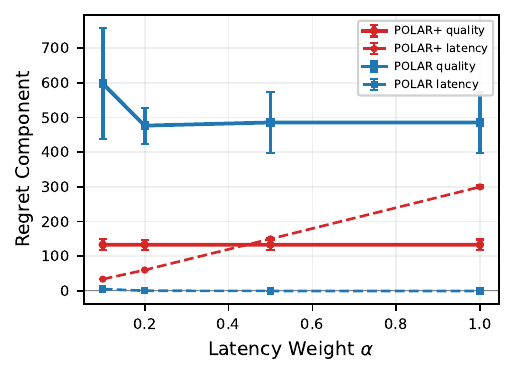}
    \Description{Regret decomposition: quality loss vs latency cost.}
       \vspace{-0.3in}
    \caption{Operational regret decomposition vs. $\alpha$. }
    \label{fig:decomp}
  \end{minipage}
    \vspace{-0.2in}
\end{figure}

\begin{figure*}[t]
  \centering
  \begin{minipage}[t]{0.19\textwidth}
    \centering
    \includegraphics[width=\textwidth]{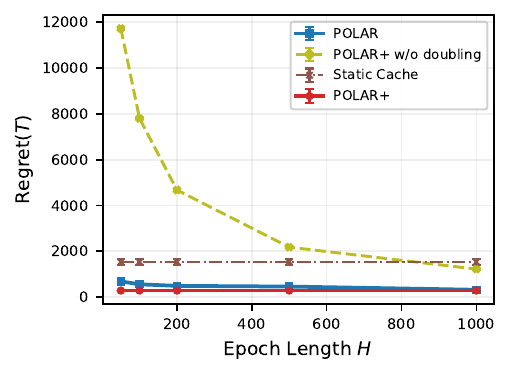}
    \Description{Regret vs epoch length H at T=100K.}
          \vspace{-0.3in}
    \caption{Regret vs. epoch length $H$.}
    \label{fig:sens_H}
  \end{minipage}\hfill
  \begin{minipage}[t]{0.19\textwidth}
    \centering
    \includegraphics[width=\textwidth]{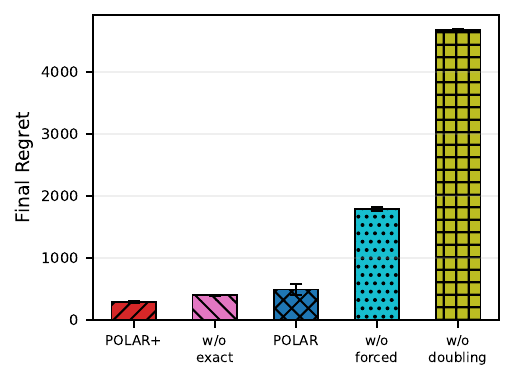}
    \Description{Final regret of each leave-one-out variant at T=100K.}
         \vspace{-0.3in}
    \caption{Leave-one-out final regret, $T=10^5$.}
    \label{fig:abl_components}
  \end{minipage}\hfill
  \begin{minipage}[t]{0.19\textwidth}
    \centering
    \includegraphics[width=\textwidth]{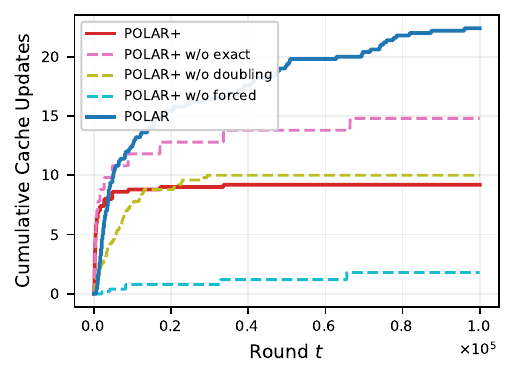}
    \Description{Cumulative cache updates vs t for each variant.}
   \vspace{-0.3in}
    \caption{Cumulative cache updates.}
    \label{fig:abl_switches}
  \end{minipage}\hfill
  \begin{minipage}[t]{0.19\textwidth}
    \centering
    \includegraphics[width=\textwidth]{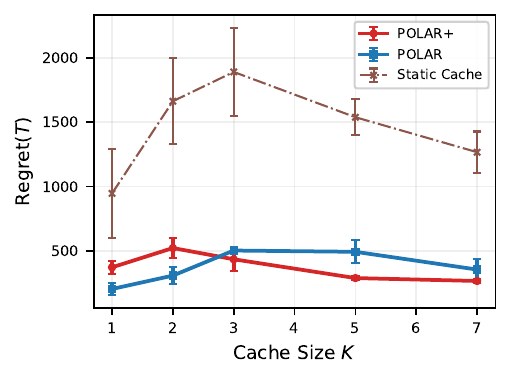}
    \Description{Regret vs cache size K.}
          \vspace{-0.3in}
    \caption{Regret vs. cache size $K$.}
    \label{fig:sens_K}
  \end{minipage}\hfill
    \begin{minipage}[t]{0.19\textwidth}
    \centering
    \includegraphics[width=\textwidth]{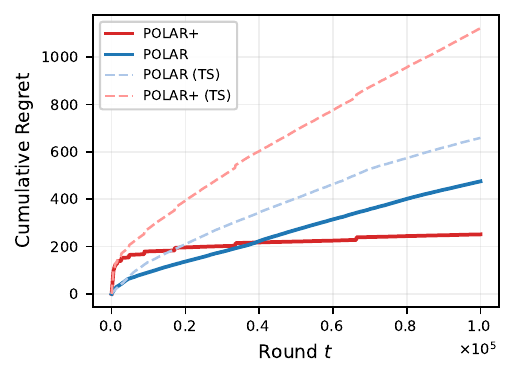}
  \Description{Ablation: UCB vs TS routing at T=100K.}
         \vspace{-0.3in}
  \caption{UCB vs TS routing, $T=10^5$.}
  \label{fig:abl_regret}
  \end{minipage}
        \vspace{-0.15in}
\end{figure*}

\subsection{Theory-Facing Validation}
\label{subsec:theory_validation}

\textbf{Sensitivity to latency weight $\alpha$.} Figure~\ref{fig:sens_alpha} varies $\alpha$ for $T=10^5$ with $\alpha\in\{0.1,0.2,0.5,1.0\}$. \textsc{POLAR}'s regret stabilizes once $\alpha\ge 0.5$: after the cache identifies valuable adapters, routing naturally avoids cold penalties. In contrast, \textsc{POLAR+}'s regret increases with $\alpha$, consistent with the $\bar B = 2+\alpha\lambda_{\max}$ dependence in Theorem~\ref{thm:ed}. Static Cache is largely insensitive to $\alpha$ because its random cache quality does not depend on the latency weight.

\textbf{Operational regret decomposition.} To connect the experiments to Theorem~\ref{thm:ed}, Figure~\ref{fig:decomp} decomposes cumulative regret into a \emph{quality loss} and a \emph{latency cost}. This is an operational decomposition rather than a literal term-by-term restatement of the theorem: quality loss reflects routing error, while latency cost captures the combined effect of cache misses and forced exploration. For \textsc{POLAR+}, the quality-loss component stays nearly constant at $\approx 134$ across all $\alpha$, confirming that the routing term $\sigma d\sqrt{NT}$ is independent of $\alpha$. Its latency cost grows approximately linearly ($34\to 300$), matching the predicted $\bar B\propto \alpha$ scaling. By contrast, \textsc{POLAR}'s latency cost is near zero because its cache-aware UCB routing naturally avoids uncached adapters, but its quality loss is $3.6\times$ higher ($486$ vs.\ $134$) due to insufficient exploration. This explains why \textsc{POLAR+} overtakes \textsc{POLAR} at larger $\alpha$: the cost of forced exploration grows, but the resulting improvement in utility estimation and cache selection more than compensates.

\textbf{Effect of cache-update cadence.} Figure~\ref{fig:sens_H} studies the fixed epoch length $H\in\{50,100,200,500,1000\}$ at $T=10^5$. \textsc{POLAR+} is shown as a flat reference because it does not use a fixed epoch size; instead, its update schedule is determined by epoch doubling. The \emph{w/o doubling} variant should be read as ``\textsc{POLAR+} with fixed epochs,'' i.e., it retains forced exploration and exact cache optimization but replaces geometric epoch growth with a fixed update interval $H$. Both \textsc{POLAR} and this \emph{w/o doubling} variant improve as $H$ grows over the tested range: \textsc{POLAR} drops from $687$ at $H=50$ to $323$ at $H=1000$, and \emph{w/o doubling} drops from $11{,}725$ to $1{,}225$. Small $H$ triggers more frequent greedy-cache updates (for \textsc{POLAR}) and more frequent forced-exploration phases (for \emph{w/o doubling}), both of which add cumulative cost. Even at the best tested $H=1000$, \textsc{POLAR} still lags \textsc{POLAR+} ($323$ vs.\ $288$) and \emph{w/o doubling} remains $4\times$ worse ($1{,}225$), showing that no fixed choice of $H$ recovers the self-tuning behavior of \textsc{POLAR+}. This provides an empirical counterpart to the switching and cache-identification improvements highlighted by Theorem~\ref{thm:ed}.

\textbf{Ablation of POLAR+ components.} Theorem~\ref{thm:ed} attributes \textsc{POLAR+}'s sublinear regret to three ingredients: \emph{epoch doubling}, \emph{exact cache optimization}, and \emph{forced exploration}. We verify that each is necessary by leave-one-out ablations at $T=10^5$. Removing one component at a time yields three variants. POLAR (fixed epoch + greedy cache + no forced exploration) serves as a lower reference. Figure~\ref{fig:abl_components} reports the final regret at $T=10^5$ for each variant. \textsc{POLAR+} achieves regret $288$; removing exact cache degrades it to $402$, indicating a moderate loss from replacing exact cache optimization with greedy updates. In contrast, removing epoch doubling inflates regret to $4{,}674$ ($16\times$ worse), while removing forced exploration yields $1{,}788$ ($6\times$ worse). These two ingredients are therefore especially load-bearing in practice. \textsc{POLAR} reaches $493$: higher than \textsc{POLAR+} but lower than either broken variant, because it pays no forced-exploration cost on top of a misconfigured schedule.

\textbf{Control-plane churn.} Figure~\ref{fig:abl_switches} reports the cumulative number of cache updates over time. \textsc{POLAR+} performs only $\approx 9$ updates across $10^5$ rounds, consistent with the $\mathcal{O}(\log T)$ cache-refresh schedule induced by epoch doubling. Exact cache is inherently stable: the \emph{w/o doubling} variant also performs only $\approx 10$ updates, showing that its large regret comes mainly from repeated forced-exploration rounds rather than from additional switching. In contrast, greedy cache is noticeably more churny: \emph{w/o exact} performs $\approx 15$ updates under doubling, and \textsc{POLAR} (fixed-epoch greedy) reaches $\approx 22$ updates, consistent with its $\mathcal{O}(T/H)$-style control-plane refresh. The \emph{w/o forced} variant switches only twice: without unbiased context samples, exact cache locks onto a biased but stable suboptimum early and never revises. This explains its persistent regret despite its low switching count. Together, Figures~\ref{fig:abl_components} and~\ref{fig:abl_switches} confirm that all three ingredients are load-bearing and epoch doubling is essential not only for regret scaling but for reducing control-plane churn.

\textbf{Learning the resident set.} Theorem~\ref{thm:ed} predicts that \textsc{POLAR+} should learn a near-optimal resident set, not merely a good routing rule. Figure~\ref{fig:cache_learning} tests this claim directly using two diagnostics: (i) the Jaccard similarity between the current cache $C_{\ell(t)}$ and the hindsight-optimal cache $C_T^\star$, and (ii) the quality loss incurred by using $C_{\ell(t)}$ as a fixed cache on the next $2{,}000$ contexts. The right panel is the more direct validation of the theorem's cache-identification objective, because Jaccard overlap is a strict set-level metric and can remain below one even when two caches induce very similar utility. Under this more operational measure, \textsc{POLAR+} converges rapidly to a high-quality resident set: its quality loss drops to $\approx 2.6$ by $t\approx 10^3$, a $6\times$ reduction from the initial value at $t=200$. \textsc{POLAR} eventually reaches a comparable Jaccard overlap, but only near $t\approx 3\times 10^4$, and its steady-state quality loss remains around $5.7$, roughly twice that of \textsc{POLAR+}. LRU and LFU, which react to request frequency rather than utility, stay near-constant at low Jaccard overlap and high quality loss throughout. Thus, while Jaccard is informative, the cache-quality-loss panel more directly shows that \textsc{POLAR+} identifies a near-optimal resident set much earlier than the alternatives.

\begin{figure}[t]
  \centering
  \includegraphics[width=\columnwidth]{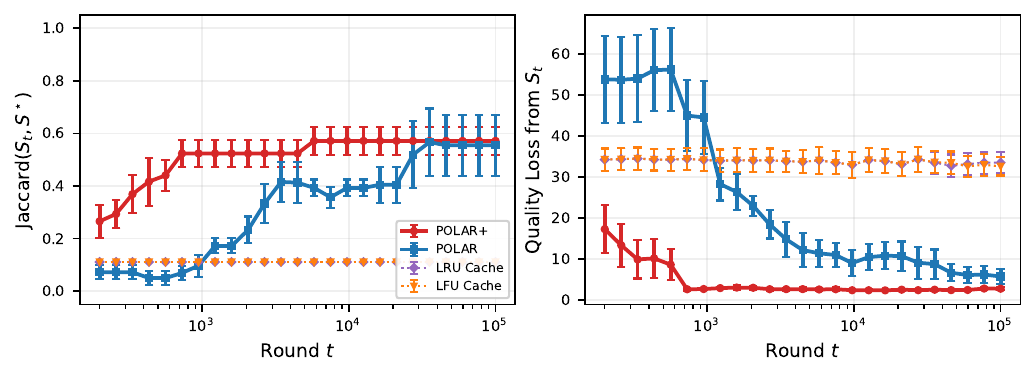}
  \Description{Jaccard(C_{ell(t)}, C_T*) and quality loss over t.}
        \vspace{-0.3in}
  \caption{Cache learning diagnostic. Left: $\mathrm{Jaccard}(C_{\ell(t)}, C_T^\star)$ over~$t$. Right: quality loss of using $C_{\ell(t)}$ on the next $2{,}000$ contexts. \textsc{POLAR+} converges rapidly to a high-quality resident set; frequency-based policies do not.}
  \label{fig:cache_learning}
        \vspace{-0.15in}
\end{figure}

\subsection{Additional Sensitivity and Ablation}
\label{subsec:additional_sensitivity}

\textbf{Cache size $K$.} Figure~\ref{fig:sens_K} varies cache size $K$ at $T=10^5$. All methods benefit from larger caches, but \textsc{POLAR+} benefits the most: its regret improves from $522$ at $K=2$ to $266$ at $K=7$ because a larger cache reduces the cost of maintaining the hot set. \textsc{POLAR} is less sensitive to $K$, as its UCB routing naturally avoids uncached adapters.

\textbf{Routing strategy.} We replace LinUCB with Linear Thompson Sampling~\cite{agrawal2013thompson} in both \textsc{POLAR} and \textsc{POLAR+}. Figure~\ref{fig:abl_regret} shows that the two routing strategies perform similarly within fixed-epoch \textsc{POLAR}, but \textsc{POLAR+} benefits more clearly from UCB. We attribute this to the more stable interaction between deterministic optimism, forced exploration, and cache identification. We view this as a router ablation rather than a central validation of the theory.

\section{Related Work}
\label{sec:related}

\textbf{LoRA adapter serving.} Recent systems work studies efficient multi-adapter LLM serving through cross-adapter batching and shared base-model execution~\cite{chen2024punica,wu2024dlora}, unified memory management and paging for adapter weights and KV caches~\cite{sheng2024slora,zhang2025fastlibra}, cold-start reduction and proactive prefetching~\cite{sui2025serverlesslora,ni2025predictive}, and edge- or hardware-aware designs for fast adapter switching~\cite{shen2025edgelora,tian2025clone}. These works optimize the \emph{systems layer} of multi-adapter serving, focusing on memory management, scheduling, and hardware efficiency. Our work operates at a different layer: given a mechanism for loading and evicting adapters, we study the online decision problem of which adapters to keep resident under uncertain, query-dependent utility.

\textbf{Contextual bandits and online routing for LLMs.} Recent work applies bandit methods to online LLM or backend selection~\cite{poon2025online,wei2025barp,li2025llmbandit,zhang2026adapter}. Closest in spirit on the learning side, Li and Li~\cite{li2026nearoptimal} study a two-timescale online deployment-and-routing problem for a streaming LLM inventory under per-query budget constraints, obtaining near-optimal $\widetilde{\mathcal{O}}(T^{2/3})$ regret. Their setting, however, is fundamentally different from ours: they consider deployment over multiple models in a non-contextual environment, whereas we study a fixed base model with a library of adapters whose context-dependent utility must be learned under a memory hierarchy with hot/cold access asymmetry, a coupling not captured by existing regret analyses. Related work also considers semantic caching and cache-assisted routing~\cite{zhu2023towards,liu2026semantic}, but these works typically cache responses or route among equally available backends. In contrast, POLAR studies a coupled problem in which request-level routing and size-$K$ adapter residency must be jointly optimized under memory constraints and paging-induced latency.

\textbf{Online caching.} Classical online paging and caching focus on competitive analysis under adversarial request sequences~\cite{sleator1985amortized,borodin2005online}, while more recent work studies learning-augmented or regret-minimizing caching policies under stochastic demand~\cite{paschos2020online,regretoptimal2024,cmab2024caching}. Xu et al.~\cite{xu2025mobile} use deep reinforcement learning for mobile-edge LLM caching and inference offloading. These methods target content objects, responses, or full models with known or externally specified value. They do not model the setting studied here, where adapter utility is both \emph{context-dependent} and \emph{unknown}, and where caching and routing are statistically coupled because the cache changes which adapters are cheap to explore and the router changes which adapters are observed.

\section{Conclusion}
\label{sec:conclusion}

We introduced POLAR, a two-timescale online learning framework for joint adapter routing and cache placement in edge LLM serving. POLAR combines a cache-aware router with an epoch-based cache controller, and admits both worst-case and sublinear regret guarantees under different operating regimes. Our analysis shows that memory-residency constraints can be handled without fundamentally slowing routing learning. Experiments with real LoRA adapters and measured paging latencies support the theory and show clear improvements over non-adaptive baselines.

\bibliographystyle{ACM-Reference-Format}
\bibliography{references}

@inproceedings{abbasi2011improved,
  title        = {Improved Algorithms for Linear Stochastic Bandits},
  author       = {Yasin Abbasi{-}Yadkori and
                  D{\'{a}}vid P{\'{a}}l and
                  Csaba Szepesv{\'{a}}ri},
  booktitle={Advances in Neural Information Processing Systems 24},
  pages={2312--2320},
  year={2011}
}

@article{tropp2012user,
  title={User-friendly tail bounds for sums of random matrices},
  author={Tropp, Joel A},
  journal={Foundations of computational mathematics},
  volume={12},
  number={4},
  pages={389--434},
  year={2012},
  publisher={Springer}
}

@inproceedings{
hu2021lora,
title={Lo{RA}: Low-Rank Adaptation of Large Language Models},
author={Edward J Hu and yelong shen and Phillip Wallis and Zeyuan Allen-Zhu and Yuanzhi Li and Shean Wang and Lu Wang and Weizhu Chen},
booktitle={International Conference on Learning Representations},
year={2022}
}

@article{sheng2024slora,
  title={S-lora: Serving thousands of concurrent lora adapters},
  author={Sheng, Ying and Cao, Shiyi and Li, Dacheng and Hooper, Coleman and Lee, Nicholas and Yang, Shuo and Chou, Christopher and Zhu, Banghua and Zheng, Lianmin and Keutzer, Kurt and others},
  journal={arXiv preprint arXiv:2311.03285},
  year={2023}
}

@article{zhang2025fastlibra,
  title={Improving the serving performance of multi-lora large language models via efficient lora and kv cache management},
  author={Zhang, Hang and Shi, Jiuchen and Wang, Yixiao and Chen, Quan and Shan, Yizhou and Guo, Minyi},
  journal={arXiv preprint arXiv:2505.03756},
  year={2025}
}

@inproceedings{ni2025predictive,
  title={Predictive-LoRA: A proactive and fragmentation-aware serverless inference system for LLMs},
  author={Ni, Yinan and Yang, Xiao and Tang, Yuqi and Qiu, Zhimin and Wang, Chen and Yuan, Tingzhou},
  booktitle={Proceedings of the 6th International Conference on Computer Science and Management Technology},
  pages={1267--1273},
  year={2025}
}

@inproceedings{shen2025edgelora,
  title={EdgeLoRA: An Efficient Multi-Tenant LLM Serving System on Edge Devices},
  author={Shen, Zheyu and He, Yexiao and Wang, Ziyao and Zhang, Yuning and Sun, Guoheng and Ye, Wanghao and Li, Ang},
  booktitle={Proceedings of the 23rd Annual International Conference on Mobile Systems, Applications and Services},
  pages={138--153},
  year={2025}
}

@inproceedings{tian2025clone,
  title={CLONE: Customizing LLMs for Efficient Latency-Aware Inference at the Edge},
  author={Tian, Chunlin and Qin, Xinpeng and Tam, Kahou and Li, Li and Wang, Zijian and Zhao, Yuanzhe and Zhang, Minglei and Xu, Chengzhong},
  booktitle={2025 USENIX Annual Technical Conference},
  pages={563--585},
  year={2025}
}

@inproceedings{poon2025online,
  title={Online multi-llm selection via contextual bandits under unstructured context evolution},
  author={Poon, Manhin and Dai, XiangXiang and Liu, Xutong and Kong, Fang and Lui, John CS and Zuo, Jinhang},
  booktitle={Proceedings of the AAAI Conference on Artificial Intelligence},
  volume={40},
  number={29},
  pages={24855--24863},
  year={2026}
}

@article{wei2025barp,
  title={Learning to Route LLMs from Bandit Feedback: One Policy, Many Trade-offs},
  author={Wei, Wang and Yang, Tiankai and Chen, Hongjie and Zhao, Yue and Dernoncourt, Franck and Rossi, Ryan A and Eldardiry, Hoda},
  journal={arXiv preprint arXiv:2510.07429},
  year={2025}
}

@article{li2025llmbandit,
  author       = {Yang Li},
  title        = {{LLM} Bandit: Cost-Efficient {LLM} Generation via Preference-Conditioned Dynamic Routing},
  journal       = {arXiv preprint arXiv:2502.02743},
  year         = {2025}
}

@inproceedings{li2026nearoptimal,
  title={Near-Optimal Online Deployment and Routing for Streaming {LLM}s},
  author={Li, Shaoang and Li, Jian},
  booktitle={The Fourteenth International Conference on Learning Representations},
  year={2026}
}

@article{sleator1985amortized,
  title={Amortized efficiency of list update and paging rules},
  author={Sleator, Daniel D and Tarjan, Robert E},
  journal={Communications of the ACM},
  volume={28},
  number={2},
  pages={202--208},
  year={1985}
}

@book{borodin2005online,
  title={Online computation and competitive analysis},
  author={Borodin, Allan and El-Yaniv, Ran},
  year={2005},
  publisher={cambridge university press}
}

@article{paschos2020online,
  title={Online convex optimization for caching networks},
  author={Paschos, Georgios S and Destounis, Apostolos and Iosifidis, George},
  journal={IEEE/ACM Transactions on Networking},
  volume={28},
  number={2},
  pages={625--638},
  year={2020}
}

@article{regretoptimal2024,
  title={Minimizing edge caching service costs through regret-optimal online learning},
  author={Quan, Guocong and Eryilmaz, Atilla and Shroff, Ness B},
  journal={IEEE/ACM Transactions on Networking},
  volume={32},
  number={5},
  pages={4349--4364},
  year={2024}
}

@article{xu2025mobile,
  title={Serving long-context LLMs at the mobile edge: Test-time reinforcement learning-based model caching and inference offloading},
  author={Xu, Minrui and Niyato, Dusit and Brinton, Christopher G},
  journal={IEEE Transactions on Networking},
  year={2026},
  publisher={IEEE}
}

@article{cmab2024caching,
  title={Online optimal service caching for multi-access edge computing: A constrained multi-armed bandit optimization approach},
  author={Chu, Weibo and Zhang, Xiaoyan and Jia, Xinming and Lui, John CS and Wang, Zhiyong},
  journal={Computer Networks},
  volume={246},
  pages={110395},
  year={2024}
}

@article{qwen2.5,
  author       = {An Yang and
                  Baosong Yang and
                  Beichen Zhang and
                  Binyuan Hui and
                  Bo Zheng and
                  Bowen Yu and
                  Chengyuan Li and
                  Dayiheng Liu and
                  Fei Huang and
                  Haoran Wei and
                  Huan Lin and
                  Jian Yang and
                  Jianhong Tu and
                  Jianwei Zhang and
                  Jianxin Yang and
                  Jiaxi Yang and
                  Jingren Zhou and
                  Junyang Lin and
                  Kai Dang and
                  Keming Lu and
                  Keqin Bao and
                  Kexin Yang and
                  Le Yu and
                  Mei Li and
                  Mingfeng Xue and
                  Pei Zhang and
                  Qin Zhu and
                  Rui Men and
                  Runji Lin and
                  Tianhao Li and
                  Tingyu Xia and
                  Xingzhang Ren and
                  Xuancheng Ren and
                  Yang Fan and
                  Yang Su and
                  Yichang Zhang and
                  Yu Wan and
                  Yuqiong Liu and
                  Zeyu Cui and
                  Zhenru Zhang and
                  Zihan Qiu},
  title        = {Qwen2.5 Technical Report},
  journal      = {arXiv preprint arXiv:2412.15115},
  year         = {2024}
}

@inproceedings{agrawal2013thompson,
  title  = {Thompson Sampling for Contextual Bandits with Linear Payoffs},
  author = {Shipra Agrawal and Navin Goyal},
  booktitle = {Proceedings of the 30th International Conference on Machine Learning},
  pages  = {127--135},
  year   = {2013}
}

@article{zhang2026adapter,
  title={Adapter-augmented bandits for online multi-constrained multi-modal inference scheduling},
  author={Zhang, Xianzhi and Xu, Yue and Zhu, Yinlin and Wu, Di and Zhou, Yipeng and Hu, Miao and Quan, Guocong},
  journal={arXiv preprint arXiv:2603.06403},
  year={2026}
}

@inproceedings{chen2024punica,
  author       = {Lequn Chen and
                  Zihao Ye and
                  Yongji Wu and
                  Danyang Zhuo and
                  Luis Ceze and
                  Arvind Krishnamurthy},
  title        = {Punica: Multi-Tenant LoRA Serving},
  booktitle    = {Proceedings of the Seventh Annual Conference on Machine Learning and Systems},
  year         = {2024}
}

@inproceedings{wu2024dlora,
  title={dLoRA: Dynamically orchestrating requests and adapters for LoRALLM serving},
  author={Wu, Bingyang and Zhu, Ruidong and Zhang, Zili and Sun, Peng and Liu, Xuanzhe and Jin, Xin},
  booktitle={18th USENIX Symposium on Operating Systems Design and Implementation (OSDI)},
  pages={911--927},
  year={2024}
}

@article{li2024caraserve,
  title={Caraserve: Cpu-assisted and rank-aware lora serving for generative llm inference},
  author={Li, Suyi and Lu, Hanfeng and Wu, Tianyuan and Yu, Minchen and Weng, Qizhen and Chen, Xusheng and Shan, Yizhou and Yuan, Binhang and Wang, Wei},
  journal={arXiv preprint arXiv:2401.11240},
  year={2024}
}

@article{sui2025serverlesslora,
  title={ServerlessLoRA: Minimizing Latency and Cost in Serverless Inference for LoRA-Based LLMs},
  author={Sui, Yifan and Wang, Hao and Yu, Hanfei and Hu, Yitao and Li, Jianxun},
  journal={arXiv preprint arXiv:2505.14468},
  year={2025}
}

@inproceedings{
zhu2023towards,
title={Towards Optimal Caching and Model Selection for Large Model Inference},
author={Banghua Zhu and Ying Sheng and Lianmin Zheng and Clark Barrett and Michael Jordan and Jiantao Jiao},
booktitle={Thirty-seventh Conference on Neural Information Processing Systems},
year={2023}
}

@inproceedings{liu2026semantic,
title={Semantic Caching for Low-Cost {LLM} Serving: From Offline Learning to Online Adaptation},
author={Liu, Xutong and Atalar, Baran and Dai, XiangXiang and Zuo, Jinhang and Wang, Siwei and Lui, John C.S. and Chen, Wei and Joe-Wong, Carlee},
booktitle={IEEE International Conference on Computer Communications (INFOCOM)},
year={2026}
}

@inproceedings{chu2011contextual,
title={Contextual Bandits with Linear Payoff Functions},
author={Chu, Wei and Li, Lihong and Reyzin, Lev and Schapire, Robert E.},
booktitle={Proceedings of the 14th International Conference on Artificial Intelligence and Statistics (AISTATS)},
pages={208--214},
year={2011}
}

\clearpage
\appendix
\clearpage
\onecolumn

\section{Notation Summary}
\label{app:notation}

Table~\ref{tab:notation} collects the notation used throughout the paper, grouped by role: system model and benchmark, structural assumptions, algorithm parameters, LinUCB router quantities, and proof-specific constructs introduced in Appendices~\ref{app:polar} and~\ref{app:polar_plus}.

\begin{table}[!ht]
\centering
\caption{Summary of notation.}
\label{tab:notation}
\small
\begin{tabular}{@{}ll@{}}
\toprule
\textbf{Symbol} & \textbf{Meaning} \\
\midrule
\multicolumn{2}{@{}l}{\emph{Sets, indices, and horizons.}} \\
$N,K,T,d$ & library size, cache size, time horizon, context dimension \\
$\mathcal{M}=\{1,\dots,N\}$ & adapter library \\
$t,\ell,\ell(t)$ & round index, epoch index, and epoch containing round $t$ \\
$L$ & $\lceil \log_2 T \rceil$, number of epochs used by \textsc{POLAR+} \\
\midrule
\multicolumn{2}{@{}l}{\emph{Contexts, adapters, and rewards.}} \\
$x_t\in\mathbb{R}^d,\ a_t\in\mathcal{M}$ & context (with $\|x_t\|_2\le 1$) and selected adapter at round $t$ \\
$\theta_a^\star\in\mathbb{R}^d$ & true parameter of adapter $a$, with $\|\theta_a^\star\|_2\le 1$ \\
$q_t(a),\eta_t(a),\sigma$ & observed quality~\eqref{eq:linmodel}, conditionally $\sigma$-sub-Gaussian noise, and its parameter \\
$\mu_t(a;C)$ & noiseless reward $\langle\theta_a^\star,x_t\rangle-\alpha\lambda_a\mathbf{1}\{a\notin C\}$ \\
$r_t(a;C)$ & realized reward $q_t(a)-\alpha\lambda_a\mathbf{1}\{a\notin C\}$ \\
$\lambda_a,\lambda_{\max}$ & cold-path latency of adapter $a$; $\lambda_{\max}=\max_{a\in\mathcal{M}}\lambda_a$ \\
$\alpha,\gamma$ & latency/quality trade-off weight; switching cost per admitted adapter \\
$\bar B$ & $2+\alpha\lambda_{\max}$, uniform per-round pseudo-regret bound \\
\midrule
\multicolumn{2}{@{}l}{\emph{Cache, benchmark, and regret.}} \\
$C_\ell\subseteq\mathcal{M}$ & resident set in epoch $\ell$, with $|C_\ell|\le K$ \\
$C_T^\star,C^\dagger$ & hindsight-optimal and population-optimal fixed caches; $C^\dagger=\arg\max_{|C|=K}F(C)$ \\
$F(C),\hat F_n(S)$ & expected cache utility $\mathbb{E}_x[\max_a\mu(a;C;x)]$ and its empirical counterpart from $n$ samples \\
$\hat\mu(a;S;x)$ & $\langle\hat\theta_a,x\rangle-\alpha\lambda_a\mathbf{1}\{a\notin S\}$ \\
$R^\star(T),\mathrm{Regret}(T)$ & oracle value and pseudo-regret (Eq.~\eqref{eq:regret}) \\
\midrule
\multicolumn{2}{@{}l}{\emph{Stochastic regularity and cacheable diversity.}} \\
$\mathcal{D},\sigma_{\min}$ & context distribution and minimum eigenvalue in Assumption~\ref{ass:iid} \\
$\mathcal{M}_{\mathrm{hot}}\subseteq\mathcal{M}$ & cacheable hot subset, with $|\mathcal{M}_{\mathrm{hot}}|\ge K$ \\
$X_a,p_{\min}$ & certified region for hot adapter $a$ and its coverage probability $\Pr_x(x\in X_a)\ge p_{\min}$ \\
$\Delta_{\min},\Delta_{\mathrm{cache}}$ & routing margin on $X_a$; cache-separability gap \\
$\Delta_{\mathrm{dom}}$ & hot-dominance gap, $\alpha\min_{a\notin\mathcal{M}_{\mathrm{hot}}}\lambda_a-2$ \\
\midrule
\multicolumn{2}{@{}l}{\emph{Algorithm parameters.}} \\
$H$ & fixed epoch length used by \textsc{POLAR} \\
$\lambda\ge 1,\delta\in(0,1)$ & ridge regularizer; target failure probability \\
$\kappa>0,c_0$ & forced-exploration constant (\textsc{POLAR+}); $c_0=\lceil\log(6Nd/\delta)\rceil$ \\
$m_\ell,M_\ell$ & per-arm forced plays $\kappa\,d(\ell+c_0)$ in epoch $\ell$; cumulative $\sum_{j\le\ell}m_j$ \\
$n_\ell,H_\ell$ & rounds completed before epoch $\ell$, $\sum_{j<\ell}(Nm_j+2^j)$; epoch-$\ell$ exploitation length $\min\{2^\ell,T-n_\ell-Nm_\ell\}$ \\
$T_E,|F|$ & total exploitation $\sum_\ell H_\ell$ and forced-exploration $\sum_\ell Nm_\ell$ rounds \\
$\textsc{SolveCache}$ & cache-optimization primitive ($\arg\max_{|S|=K}\hat F_n(S)$ in the analysis) \\
\midrule
\multicolumn{2}{@{}l}{\emph{LinUCB router quantities.}} \\
$V_a(V_{a,t}),b_a,\hat\theta_a(\hat\theta_{a,t})$ & design matrix $\lambda I_d+\sum_{s<t}\mathbf{1}\{a_s=a\}x_sx_s^\top$, response $\sum_{s<t}\mathbf{1}\{a_s=a\}q_s(a)x_s$, estimate $V_{a,t}^{-1}b_a$ \\
$s_{a,t},\beta_t$ & prediction width $\sqrt{x_t^\top V_{a,t}^{-1}x_t}$; LinUCB confidence radius (Lemma~\ref{lem:ucb}) \\
$\mathrm{score}_t(a)$ & cache-aware UCB score $\langle\hat\theta_a,x_t\rangle+\beta_t s_{a,t}-\alpha\lambda_a\mathbf{1}\{a\notin C_{\ell(t)}\}$ \\
\midrule
\multicolumn{2}{@{}l}{\emph{Proof-specific quantities (\textsc{POLAR+}).}} \\
$\ell_0,\ell_{\mathrm{cache}},\ell_\star$ & bootstrap epochs for the bonus (Lemma~\ref{lem:bootstrap}) and cache-identification (Lemma~\ref{lem:cachebt}) conditions; $\ell_\star=\max(\ell_0,\ell_{\mathrm{cache}})$ \\
$M_{a,\ell}^{\mathrm{pre}},\widetilde M_{a,\ell}$ & pre-exploitation plays of $a$ in epoch $\ell$; full-dimensional plays (forced plus $X_a$-certified exploitation) \\
$I_t^a$ & indicator $\mathbf{1}\{x_t\in X_a\}$ \\
$\hat\theta_{a,\ell}^{\mathrm{pre}}$ & ridge estimate from the $M_{a,\ell}^{\mathrm{pre}}$ pre-exploitation samples \\
$\varepsilon_\ell,\varepsilon_{\mathrm{all}}$ & $\max_{a\in\mathcal{M}_{\mathrm{hot}}}\|\hat\theta_{a,\ell}^{\mathrm{pre}}-\theta_a^\star\|_2$; $\max_{a\in\mathcal{M}}\|\hat\theta_a-\theta_a^\star\|_2$ \\
$\rho_n$ & uniform cache-utility deviation bound (Lemma~\ref{lem:cacheid}) \\
$F_{\mathrm{hot}}(S),\hat F_n^{\mathrm{hot}}(S)$ & hot-restricted cache utility $\mathbb{E}_x[\max_{a\in\mathcal{M}_{\mathrm{hot}}}\mu(a;S;x)]$ and its empirical counterpart \\
$c_1,c_2,c_3$ & absolute constants in the regret bound~\eqref{eq:refined-bound} \\
\bottomrule
\end{tabular}
\end{table}

\clearpage
\twocolumn

\section{Greedy cache utility maximization procedure}\label{app:greedy}

\begin{algorithm}[h]
\caption{\textsc{GreedyCacheUpdate}$(C_{\mathrm{prev}},\{\hat\theta_a,V_a\},\mathcal{D})$}
\label{alg:greedy_cache}
\begin{algorithmic}[1]
\STATE Precompute
$\mathrm{UCB}_t(a):=
\langle \hat\theta_a,x_t\rangle
+\beta_t\sqrt{x_t^\top V_a^{-1}x_t}$
for all $(t,a)$ in $\mathcal{D}$
\FOR{each $t$ in $\mathcal{D}$}
  \STATE $b_t\leftarrow \max_{a}
  (\mathrm{UCB}_t(a)-\alpha\lambda_a)$
  \hfill\COMMENT{all-cold baseline}
  \STATE $\Delta_t(a)\leftarrow
  \max\{0,\mathrm{UCB}_t(a)-b_t\}$ for all $a$
\ENDFOR
\STATE $C\leftarrow \emptyset$; \ $\mathrm{best}_t\leftarrow 0$ for all $t$
\FOR{$i=1$ to $K$}
  \STATE For each $a\notin C$, compute
  \[
    g(a)=
    \sum_t \max\{0,\Delta_t(a)-\mathrm{best}_t\}
    -\gamma \mathbf{1}\{a\notin C_{\mathrm{prev}}\}
  \]
  \STATE $a^\star\leftarrow \arg\max_{a\notin C} g(a)$;
  \textbf{break} if $g(a^\star)\le 0$
  \STATE $C\leftarrow C\cup\{a^\star\}$ and update $\mathrm{best}_t$
\ENDFOR
\RETURN $C$
\end{algorithmic}
\end{algorithm}

\section{Proof of Theorem~\ref{thm:wc} (POLAR)}
\label{app:polar}

We first state the supporting lemmas, then give the proof.

\begin{lemma}[UCB confidence ellipsoid]\label{lem:ucb}
Assume $\|\theta_a^\star\|_2\le 1$, $\|x_t\|_2\le 1$, and $\eta_t(a)$ is conditionally $\sigma$-sub-Gaussian. Then with probability $\ge 1-\delta$, for all $a\in\mathcal{M}$ and $t\le T$,
\[
  |x_t^\top(\hat\theta_{a,t}-\theta_a^\star)|\le\beta_t\,s_{a,t},
\]
where $\beta_t:=\sigma\sqrt{d\log(1+\tfrac{t}{d\lambda})+2\log\tfrac{N}{\delta}}+\sqrt{\lambda}$. Moreover, $\beta_t\ge 0$ and $\beta_{t_1}\le\beta_{t_2}$ whenever $t_1\le t_2$.
\end{lemma}

\begin{proof}
\emph{Step 1 (Per-arm bound).} Fix arm $a$. Let $S_{a,t}:=\sum_{s<t}\mathbf{1}\{a_s=a\}\eta_s(a)x_s$. Apply the self-normalized bound for vector-valued martingales~\cite[Theorem~1]{abbasi2011improved} with failure probability $\delta'=\delta/N$:
\[
\|S_{a,t}\|_{V_{a,t}^{-1}}
\le\sigma\sqrt{2\log\frac{\det(V_{a,t})^{1/2}\det(\lambda I)^{-1/2}}{\delta/N}}.
\]
Since $V_{a,t}(\hat\theta_{a,t}-\theta_a^\star)=S_{a,t}-\lambda\theta_a^\star$, the triangle inequality gives $\|\hat\theta_{a,t}-\theta_a^\star\|_{V_{a,t}} \le \|S_{a,t}\|_{V_{a,t}^{-1}} + \sqrt{\lambda}\|\theta_a^\star\|_2 \le \|S_{a,t}\|_{V_{a,t}^{-1}} + \sqrt{\lambda}$.

\emph{Step 2 (Log-det simplification).} AM--GM on eigenvalues gives $\det(V_{a,t})/\det(\lambda I)\le(1{+}t/(d\lambda))^d$. Hence the log-det term is at most $\tfrac{d}{2}\log(1{+}t/(d\lambda))$; substituting yields~$\beta_t$.

\emph{Step 3 (Union bound).} Taking a union over $N$ arms, $\Pr(\exists a:\text{bound fails})\le N\cdot\delta/N=\delta$. The self-normalized bound~\cite{abbasi2011improved} holds uniformly over all $t\ge 0$, so no additional union bound over time is needed.

\emph{Step 4 (Monotonicity).} $t_1\le t_2$ implies $\beta_{t_1}\le\beta_{t_2}$ by monotonicity of $\log$ and $\sqrt{\cdot}$.

\emph{Step 5 (Consequent).} Cauchy--Schwarz in the $V$-norm: $|x^\top u|\le\|u\|_V\sqrt{x^\top V^{-1}x}$.
\end{proof}

\begin{lemma}[Routing regret bound]\label{lem:routing}
On the event of Lemma~\ref{lem:ucb}, with $\lambda\ge 1$,
\[
  \sum_{t=1}^T\bigl[\max_a\mu_t(a;C_{\ell(t)})-\mu_t(a_t;C_{\ell(t)})\bigr]
  \le 2\beta_T\sqrt{2NdT\log(1+\tfrac{T}{d\lambda})}.
\]
\end{lemma}

\begin{proof}
On the confidence event, the LinUCB score $\langle\hat\theta_a,x_t\rangle+\beta_t s_{a,t} -\alpha\lambda_a\mathbf{1}\{a\notin C_{\ell(t)}\}$ upper-bounds $\mu_t(a;C_{\ell(t)})$ for all $a$. Since $a_t$ maximizes the score, the per-round gap satisfies $\Delta_t\le 2\beta_t\,s_{a_t,t}\le 2\beta_T\,s_{a_t,t}$. By Cauchy--Schwarz, $\sum_t\Delta_t\le 2\beta_T\sqrt{T\sum_t s_{a_t,t}^2}$. The elliptic potential lemma~\cite[Appendix~C]{abbasi2011improved} gives $\sum_t s_{a_t,t}^2 \le 2Nd\log(1+T/(d\lambda))$.
\end{proof}

\begin{lemma}[Submodularity of cache benefit]\label{lem:submodular}
The cache benefit $f(S):=\sum_{t=1}^H[\max_a\mu_t(a;S)-\max_a\mu_t(a;\emptyset)]$ is monotone non-decreasing and submodular.
\end{lemma}

\begin{proof}
Write $\mu_t(a;S)=v_t(a)-\mathrm{pen}(a)\cdot\mathbf{1}\{a\notin S\}$ where $v_t(a):=\langle\theta_a^\star,x_t\rangle$ and $\mathrm{pen}(a):=\alpha\lambda_a\ge 0$.

\emph{Pointwise monotonicity.} $S\subseteq S'$ implies $\mu_t(a;S)\le\mu_t(a;S')$ for all $a$, hence $\max_a\mu_t(a;S)\le\max_a\mu_t(a;S')$.

\emph{Single-step submodularity.} Fix $S\subseteq S'$ and $a_0$; let $b$ and $c$ be the maximizers of $\mu_t(\cdot;S'\cup\{a_0\})$ and $\mu_t(\cdot;S)$, respectively.

\textbf{Case $b=a_0$:} $\mu_t(b;S'\cup\{a_0\})=v_t(a_0) \le\max_a\mu_t(a;S\cup\{a_0\})$ (since $a_0\in S\cup\{a_0\}$), and $\max_a\mu_t(a;S')\ge\mu_t(c;S')\ge\mu_t(c;S)$.

\textbf{Case $b\ne a_0$:} $\max_a\mu_t(a;S')\ge\mu_t(b;S')=\mu_t(b;S'\cup\{a_0\})$, and $\max_a\mu_t(a;S\cup\{a_0\})\ge\mu_t(c;S)$.

Both cases yield
\begin{align*}
&\max_a\mu_t(a;S'\!\cup\!\{a_0\})+\max_a\mu_t(a;S)\\
&\quad\le\max_a\mu_t(a;S\!\cup\!\{a_0\})+\max_a\mu_t(a;S').
\end{align*}
Summing over $t$ gives submodularity.
\end{proof}

\begin{proof}[Proof of Theorem \ref{thm:wc}.]
Conditioning on the event of Lemma~\ref{lem:ucb}, we write
\begin{align*}
\mathrm{Regret}(T)
&=\sum_t[\max_a\mu_t(a;C_T^\star)-\mu_t(a_t;C_{\ell(t)})]\\
&\quad+\sum_{\ell\ge 2}\gamma|C_\ell\setminus C_{\ell-1}|.
\end{align*}
Adding and subtracting $\max_a\mu_t(a;C_{\ell(t)})$ gives $\mathrm{Regret}(T)=\text{(I)}+\text{(II)}+\text{(III)}$.

\textbf{(I)} is bounded by Lemma~\ref{lem:routing}.

\textbf{(II):} By monotonicity (Lemma~\ref{lem:submodular}),
\begin{align*}
&\max_a\mu_t(a;C_T^\star)-\max_a\mu_t(a;C_{\ell(t)})\\
&\quad\le\max_a\mu_t(a;C_T^\star)-\max_a\mu_t(a;\emptyset)
\le\alpha\lambda_{\max},
\end{align*}
since removing all cached arms costs at most $\alpha\lambda_{\max}$ per round. Summing gives $\alpha\lambda_{\max}T$.

\textbf{(III):} At most $\lceil T/H\rceil$ epoch boundaries, each costing $\le K\gamma$.
\end{proof}

\section{Proof of Theorem~\ref{thm:ed} (POLAR+)}
\label{app:polar_plus}

\begin{lemma}[Forced exploration controls bonus]\label{lem:bonus}
Under~\ref{ass:iid}, after $r$ cumulative rounds of forced exploration (round-robin) for arm $a$, with probability $\ge 1-\delta'$: $\lambda_{\min}(V_{a})\ge\lambda+\sigma_{\min}r/2$, and consequently $\beta_t\,s_{a,t}\le\beta_t/\sqrt{\lambda+\sigma_{\min}r/2}$ for all $\|x_t\|\le 1$.
\end{lemma}

\begin{proof}
During forced exploration, arm selection is round-robin (independent of $x_t$), so arm $a$'s contexts are i.i.d.\ draws from $\mathcal{D}$. By~\ref{ass:iid}, $\mathbb{E}[x_tx_t^\top]\succeq\sigma_{\min}I_d$. A matrix Chernoff bound~\cite[Theorem~5.1.1]{tropp2012user} gives $\lambda_{\min}(\sum_{s=1}^r x_sx_s^\top)\ge\sigma_{\min}r/2$ with probability $\ge 1-\delta'$, provided $r\ge c\,d\log(d/\delta')$. Since $V_a\succeq\lambda I+\sum x_sx_s^\top$, the eigenvalue bound follows.
\end{proof}

\begin{lemma}[Bootstrap epoch]\label{lem:bootstrap}\emergencystretch=1em
Under~\ref{ass:iid}--\ref{ass:diversity}, let $M_\ell:=\sum_{j=0}^{\ell}m_j$ be the cumulative forced plays per arm through epoch~$\ell$. There exists $\ell_0$ such that for all $\ell \ge \ell_0$ and $a \in \mathcal{M}$, $\beta_t s_{a,t} < \Delta_{\min}/2$ on the confidence event.
\end{lemma}

\begin{proof}
After $\ell+1$ epochs of forced exploration (round-robin over all $N$ arms), each arm has $\ge M_\ell$ cumulative forced plays. By Lemma~\ref{lem:bonus}, the bonus is at most $\beta_T/\sqrt{\sigma_{\min}M_\ell/2}$. This is $<\Delta_{\min}/2$ when $M_\ell>8\beta_T^2/(\sigma_{\min}\Delta_{\min}^2)$. Since $M_\ell=\Theta(d\ell^2)$ grows without bound, such $\ell_0$ exists. After epoch~$\ell_0$, each subsequent epoch contributes $m_\ell$ further forced-exploration samples to every arm, so $\lambda_{\min}(V_a)$ continues to increase regardless of the exploitation policy (since $V_a+x_tx_t^\top\succeq V_a$). Because $\beta_t\le\beta_T$ remains bounded, $\beta_t s_{a,t}\le\beta_T/\sqrt{\lambda_{\min}(V_a)}$ can only decrease, and the bonus condition persists for all subsequent epochs.
\end{proof}

\begin{lemma}[Correct arm selection]\label{lem:selection}
Under~\ref{ass:diversity}, on the confidence event of Lemma~\ref{lem:ucb}, suppose every arm's exploration bonus satisfies $\beta_t\,s_{a,t}<\Delta_{\min}/2$ for all $a\in\mathcal{M}$. Then for any context $x\in X_{a^\star}$ and any hot cache $C\subseteq\mathcal{M}_{\mathrm{hot}}$ with $|C|=K$, the LinUCB score of $a^\star$ strictly exceeds that of every other arm (including cold arms).
\end{lemma}

\begin{proof}
For $a'\ne a^\star$ and $x\in X_{a^\star}$:
\begin{align*}
\mathrm{score}(a^\star)
&\ge\mu(a^\star;C;x)
  \ge\mu(a';C;x)+\Delta_{\min}\\
&>\mu(a';C;x)+2\beta\,s_{a'}
\ge\mathrm{score}(a').
\end{align*}
The first inequality uses confidence (UCB $\ge$ true mean), the second uses the margin on $X_{a^\star}$, the third uses $\Delta_{\min}>2\beta\,s_{a'}$, and the fourth uses confidence (score $\le$ true mean $+\,2\beta\,s$). The margin in~\ref{ass:diversity} quantifies over all $a'\ne a^\star$ (including cold arms), so score dominance holds with unrestricted routing over $\mathcal{M}$. In particular, for cold arms $a'\notin\mathcal{M}_{\mathrm{hot}}$ with $a'\notin C$, hot dominance gives $\mu(a';C;x)\le -1-\Delta_{\mathrm{dom}}$, which is far below any cached hot arm, so the margin condition is satisfied a fortiori.
\end{proof}

\begin{lemma}[Exploitation play count]\label{lem:explcount}
Under~\ref{ass:iid}--\ref{ass:diversity}, suppose the bonus condition of Lemma~\ref{lem:selection} holds throughout epoch~$\ell$ and the epoch-$\ell$ cache satisfies $C_\ell\subseteq\mathcal{M}_{\mathrm{hot}}$. Then with probability $\ge 1-\delta'$, each $a\in\mathcal{M}_{\mathrm{hot}}$ receives at least $p_{\min}\cdot 2^\ell/2$ exploitation plays.
\end{lemma}

\begin{proof}
By Lemma~\ref{lem:selection}, LinUCB selects $a$ on every context in $X_a$ during exploitation. Since contexts are i.i.d., the number of such contexts in $2^\ell$ rounds concentrates around $p_{\min}\cdot 2^\ell$. A Chernoff bound gives at least $p_{\min}\cdot 2^\ell/2$ with probability $\ge 1-\delta'$.
\end{proof}

\begin{lemma}[Uniform cache deviation]\label{lem:cacheid}
Let $\varepsilon=\max_a\|\hat\theta_a-\theta_a^\star\|_2$, $F(S):=\mathbb{E}_x[\max_a\mu(a;S;x)]$, and $\hat F_n(S):=\frac{1}{n}\sum_{t=1}^n\max_a\hat\mu(a;S;x_t)$. Define $\rho_n:=\varepsilon+\bar B\sqrt{2(K\log(eN/K)+\log(2/\delta'))/n}$. Then with probability $\ge 1-\delta'$,
\[
  \sup_{|S|\le K}|\hat F_n(S)-F(S)|\le\rho_n.
\]
Consequently, if $\hat C\in\arg\max_{|S|\le K}\hat F_n(S)$, then for any fixed cache $C_0$ with $|C_0|=K$, $F(C_0)-F(\hat C)\le 2\rho_n$.
\end{lemma}

\begin{proof}
\emph{Step 1.} $|\max_a\hat\mu(a;S;x)-\max_a\mu(a;S;x)|\le\varepsilon$ for all $S,x$ (pointwise Lipschitz of max).

\emph{Step 2.} For fixed $S$, $Y_t:=\max_a\mu(a;S;x_t)$ are i.i.d.\ with $|Y_t|\le\bar B$. By Hoeffding with per-cache failure probability $\delta'_0:=\delta'/(eN/K)^K$, $|\frac{1}{n}\sum_t Y_t-F(S)|\le\bar B\sqrt{2\log(2/\delta'_0)/n}$.

\emph{Step 3.} Union over at most $(eN/K)^K$ caches of size $\le K$; the total failure probability is at most $(eN/K)^K\cdot\delta'_0=\delta'$.

\emph{Step 4 (uniform deviation).} Combining steps 1--3: $|\hat F_n(S)-F(S)|\le\rho_n$ for all $|S|\le K$.

\emph{Step 5 (optimizer gap).} $F(C_0)\le\hat F_n(C_0)+\rho_n \le\hat F_n(\hat C)+\rho_n\le F(\hat C)+2\rho_n$.
\end{proof}

\begin{lemma}[Cache bootstrap]\label{lem:cachebt}
Under~\ref{ass:iid}--\ref{ass:diversity}, there exists $\ell_{\mathrm{cache}}$ with $M_{\ell_{\mathrm{cache}}} >\beta_T^2/(\sigma_{\min}\min(\Delta_{\mathrm{cache}},\Delta_{\mathrm{dom}})^2)$ such that for all $\ell\ge\ell_{\mathrm{cache}}$: \emph{(i)} the algorithm's cache $C_\ell:=\arg\max_{|S|=K}\hat F_{n_\ell+Nm_\ell}(S)$ satisfies $C_\ell\subseteq\mathcal{M}_{\mathrm{hot}}$, and \emph{(ii)} for every hot cache $S\subseteq\mathcal{M}_{\mathrm{hot}}$, the estimated max $\max_a\hat\mu(a;S;x)$ is achieved by a hot arm (estimated hot dominance), both with high probability.
\end{lemma}

\begin{proof}
After $\ell+1$ epochs of forced exploration, every arm $a\in\mathcal{M}$ has $\ge M_\ell$ plays. By Lemma~\ref{lem:ucb}, $\|\hat\theta_a-\theta_a^\star\|_{V_a} \le\beta_T$. By Lemma~\ref{lem:bonus}, $\lambda_{\min}(V_a)\ge\sigma_{\min}M_\ell/2$. Dividing:
\[
  \|\hat\theta_a-\theta_a^\star\|_2
  \le\frac{\beta_T}{\sqrt{\sigma_{\min}M_\ell/2}}
\]
for all $a$. Let $\varepsilon_{\mathrm{all}}:=\max_{a\in\mathcal{M}} \|\hat\theta_a-\theta_a^\star\|_2$. \emph{(i) Cache separability.} By Lemma~\ref{lem:cacheid} with $n=n_\ell+Nm_\ell$ (total observations when $C_\ell$ is selected), writing $\tilde n_\ell:=n_\ell+Nm_\ell$, $\sup_{|S|\le K}|\hat F_{\tilde n_\ell}(S)-F(S)|\le\rho_{\tilde n_\ell}$ where $\rho_{\tilde n_\ell}:=\varepsilon_{\mathrm{all}} +\bar B\sqrt{2(K\log(eN/K)+\log(2/\delta'))/\tilde n_\ell}$. When $2\rho_{\tilde n_\ell}<\Delta_{\mathrm{cache}}$, any $S\not\subseteq\mathcal{M}_{\mathrm{hot}}$ satisfies $\hat F_{\tilde n_\ell}(S)\le F(S)+\rho_{\tilde n_\ell} \le F(C^\dagger)-\Delta_{\mathrm{cache}}+\rho_{\tilde n_\ell} <F(C^\dagger)-\rho_{\tilde n_\ell}\le\hat F_{\tilde n_\ell}(C^\dagger)$. Therefore $C_\ell\subseteq\mathcal{M}_{\mathrm{hot}}$.

\emph{(ii) Estimated hot dominance.} Consider any hot cache $S\subseteq\mathcal{M}_{\mathrm{hot}}$ with $|S|=K$. (Since both the oracle and algorithm optimize over $|S|=K$ and $K\ge 1$, the cache is always nonempty.) By hot dominance~\ref{ass:diversity}, for any $x$ with $\|x\|\le 1$: the best cached hot arm in $S$ scores $\max_{a\in S}\langle\theta_a^\star,x\rangle\ge -1$, while every non-hot arm $a'\notin\mathcal{M}_{\mathrm{hot}}$ scores $\langle\theta_{a'}^\star,x\rangle-\alpha\lambda_{a'}\le 1-\alpha\min_{a'\notin\mathcal{M}_{\mathrm{hot}}}\lambda_{a'} =-1-\Delta_{\mathrm{dom}}$ (hot-but-not-cached arms remain in $\mathcal{M}_{\mathrm{hot}}$ and are allowed in the max). The true gap is $\ge\Delta_{\mathrm{dom}}$. When $2\varepsilon_{\mathrm{all}}<\Delta_{\mathrm{dom}}$, estimation shifts each score by at most $\varepsilon_{\mathrm{all}}$, so estimated hot arm scores still exceed estimated cold arm scores (the gap $\Delta_{\mathrm{dom}}-2\varepsilon_{\mathrm{all}}>0$ remains positive). Therefore $\max_a\hat\mu(a;S;x)=\max_{a\in\mathcal{M}_{\mathrm{hot}}} \hat\mu(a;S;x)$ for all nonempty hot caches $S$.

Both conditions hold for $M_\ell>c\,\beta_T^2/(\sigma_{\min} \min(\Delta_{\mathrm{cache}},\Delta_{\mathrm{dom}})^2)$.
\end{proof}

\begin{lemma}[Estimation error with exploitation plays]\label{lem:l2}
Under~\ref{ass:iid}--\ref{ass:diversity}, let $M_{a,\ell}^{\mathrm{pre}}$ denote the cumulative plays of arm $a$ available at the start of epoch~$\ell$'s exploitation phase, i.e., all previous epochs' forced and exploitation rounds plus the current epoch's forced exploration, but \emph{excluding} epoch~$\ell$'s exploitation. For $\ell\ge\ell_\star+1$ where $\ell_\star:=\max(\ell_0,\ell_{\mathrm{cache}})$ (Lemmas~\ref{lem:bootstrap} and~\ref{lem:cachebt}), $M_{a,\ell}^{\mathrm{pre}}\ge c\,p_{\min}\cdot 2^{\ell-1}$, and with probability $\ge 1-\delta/5$:
\[
  \varepsilon_\ell:=\max_{a\in\mathcal{M}_{\mathrm{hot}}}
  \|\hat\theta_{a,\ell}^{\mathrm{pre}}-\theta_a^\star\|_2
  \le\frac{c_1(\sigma\sqrt{d\log(NLT/\delta)}+\sqrt\lambda)}
  {\sqrt{\sigma_{\min}\,p_{\min}\cdot 2^{\ell-1}}}.
\]
Here $\hat\theta_{a,\ell}^{\mathrm{pre}}$ is the estimate computed from the $M_{a,\ell}^{\mathrm{pre}}$ samples available when $C_\ell$ is selected, ensuring no forward-looking dependence.
\end{lemma}

\begin{proof}
\emph{Cumulative plays.} Epochs $j<\ell_\star$ contribute at least $m_j$ forced plays each (exploitation play counts are not guaranteed before both bonus and cache bootstrap). Each epoch $j$ with $\ell_\star\le j<\ell$ contributes $m_j$ forced plays and $\ge p_{\min}\cdot 2^j/2$ exploitation plays (Lemma~\ref{lem:explcount}, whose hypotheses hold since $j\ge\ell_0$ gives the bonus condition and $j\ge\ell_{\mathrm{cache}}$ gives $C_j\subseteq\mathcal{M}_{\mathrm{hot}}$). Epoch~$\ell$ itself contributes $m_\ell$ forced plays (but its exploitation has not yet occurred). Hence $M_{a,\ell}^{\mathrm{pre}}\ge M_\ell +\frac{p_{\min}}{2}\sum_{j=\ell_\star}^{\ell-1}2^j \ge c\,p_{\min}\cdot 2^{\ell-1}$, using $\sum_{j=\ell_\star}^{\ell-1}2^j =2^\ell-2^{\ell_\star}\ge 2^{\ell-1}$ for $\ell\ge\ell_\star+1$.

\emph{Design matrix.} Both forced and exploitation plays from \emph{previous} epochs contribute to $V_{a,\ell}^{\mathrm{pre}}$. Define the indicator $I_t^a:=\mathbf{1}\{x_t\in X_a\}$. Since $X_a$ is a fixed set and $x_t\overset{\text{i.i.d.}}{\sim}\mathcal{D}$, the variables $\{I_t^a\}$ are i.i.d.\ Bernoulli with $\mathbb{E}[I_t^a]\ge p_{\min}$. We identify two sources of independent samples for the design matrix:

\emph{(a) Forced exploration:} Round-robin arm assignment is predetermined (independent of $\{x_t\}$), so forced-exploration contexts for arm~$a$ are i.i.d.\ from $\mathcal{D}$.

\emph{(b) Exploitation from $X_a$:} On the confidence event with bonus $<\Delta_{\min}/2$ (which holds for $j\ge\ell_\star$), Lemma~\ref{lem:selection} gives: $I_t^a=1\Rightarrow a_t=a$. Since $I_t^a$ depends only on $x_t$ (not on the algorithm's state), the samples $\{x_t:I_t^a=1\}$ form an i.i.d.\ draw from $\mathcal{D}(\cdot\mid X_a)$, with $\mathbb{E}[x_tx_t^\top\mid x_t\in X_a]\succeq\sigma_{\min}I_d$ by~\ref{ass:diversity} (full-dimensional). Moreover, $V_{a,\ell}^{\mathrm{pre}}$ includes these samples because $a_t=a$ on $\{I_t^a=1\}$. Since $X_a$ is a fixed (non-adaptive) set, thinning an i.i.d.\ sequence by the criterion $x_t\in X_a$ preserves i.i.d.-ness of the selected sub-sequence. Combined with forced-exploration samples (also i.i.d.), Tropp's matrix Chernoff for independent (not necessarily identically distributed) matrices applies.

Let $\widetilde M_{a,\ell}$ denote the number of full-dimensional plays (forced plays plus exploitation plays with $I_t^a=1$). By the coverage and play-count bounds, $\widetilde M_{a,\ell}\ge M_\ell +\frac{p_{\min}}{2}\sum_{j=\ell_\star}^{\ell-1}2^j \ge c\,p_{\min}\cdot 2^{\ell-1}$. A matrix Chernoff bound~\cite[Theorem~5.1.1]{tropp2012user} applied to the first $\lfloor c\,p_{\min}\cdot 2^{\ell-1}\rfloor$ of these independent samples (a deterministic count), each with covariance $\succeq\sigma_{\min}I_d$, gives $\lambda_{\min}(V_{a,\ell}^{\mathrm{pre}})\ge\lambda +\sigma_{\min}\widetilde M_{a,\ell}/c$; additional samples only increase $\lambda_{\min}$.

\emph{$\ell_2$ bound.} The self-normalized bound (Lemma~\ref{lem:ucb}) gives $\|\hat\theta_{a,\ell}^{\mathrm{pre}}-\theta_a^\star\| _{V_{a,\ell}^{\mathrm{pre}}}\le\beta_T$ (using $\beta_t\le\beta_T$ for all $t\le T$). Dividing by $\sqrt{\lambda_{\min}(V_{a,\ell}^{\mathrm{pre}})}$, which is ${\ge}\sqrt{\sigma_{\min}p_{\min}\cdot 2^{\ell-1}/c}$, yields the result. Union bound over $N$ arms and $L$ epochs.
\end{proof}

\begin{proof}[Proof of Theorem \ref{thm:ed}.]
Define $\mathcal{E}:=\mathcal{E}_{\mathrm{ucb}}\cap \mathcal{E}_{\mathrm{bonus}}\cap \mathcal{E}_{\mathrm{counts}}\cap \mathcal{E}_{\mathrm{cache}}\cap \mathcal{E}_{\mathrm{dev}}$, where $\mathcal{E}_{\mathrm{ucb}}$ is the confidence event of Lemma~\ref{lem:ucb} (failure $\le\delta/6$) and the remaining four are per-epoch events (bonus, play-count, cache-deviation, Hoeffding), each with per-epoch failure $\le\delta/(6\cdot 2^\ell)$ (geometric allocation; total $\le 4\cdot\delta/6$). By union bound, $\Pr(\mathcal{E})\ge 1-\delta$. All subsequent analysis is deterministic on $\mathcal{E}$.

Let $H_\ell:=\min\{2^\ell,\,T-n_\ell-Nm_\ell\}$ denote the number of exploitation rounds in epoch~$\ell$ (equal to $2^\ell$ except possibly the last truncated epoch), and let $L:=\lceil\log_2 T\rceil$. Each epoch consists of $Nm_\ell$ forced-exploration rounds followed by $H_\ell$ exploitation rounds. Let $|F|:=\sum_{\ell=0}^L Nm_\ell$ denote the total forced-exploration rounds. Since $\max_a\mu_t(a;C_T^\star)\le 1$ and $\mu_t(a_t;C_{\ell(t)})\ge -1-\alpha\lambda_{\max}$, each forced round incurs per-round regret at most $\bar B$, contributing at most $|F|\bar B$ in total. \emph{All subsequent analysis applies only to exploitation rounds.}

\textbf{Regret decomposition (exploitation rounds).} Let
\[
C_T^\star:=\arg\max_{|C|=K}\sum_t\max_a\mu_t(a;C)
\]
denote the hindsight-optimal cache. By Lemma~\ref{lem:cacheid} with $\varepsilon=0$ and failure probability $\delta':=\delta/(2c)$,
\[
\sup_{|C|=K}\Bigl|\tfrac{1}{T}\textstyle\sum_t
\max_a\mu(a;C;x_t)-F(C)\Bigr|
\le\bar B\sqrt{\tfrac{2\bigl(K\log\frac{eN}{K}
+\log\frac{2}{\delta'}\bigr)}{T}}\;=:\;\rho_T.
\]
Hence $R^\star(T)\le T\,F(C^\dagger)+T\rho_T$. Let $T_E:=\sum_\ell H_\ell$ denote the total exploitation rounds. Writing $TF(C^\dagger)=|F|\,F(C^\dagger)+T_E\,F(C^\dagger)$ and splitting by round type:
\begin{align*}
\mathrm{Regret}(T)
&\le\underbrace{[|F|\,F(C^\dagger)
-\textstyle\sum_{t\in F}\mu_t(a_t;C_{\ell(t)})]}_{\le\,|F|\bar B}\\
&\quad+\bigl[T_E\,F(C^\dagger)
-\textstyle\sum_{t\in E}\mu_t(a_t;C_\ell)\bigr]
+T\rho_T+\text{switching}.
\end{align*}
The exploitation-round contribution satisfies
\begin{align*}
  \mathrm{Regret}_E
  &\le\underbrace{\sum_{t\in E}[\max_a\mu_t(a;C_\ell)-\mu_t(a_t;C_\ell)]}_{\text{routing}}\\
  &\quad+\underbrace{T_E\,F(C^\dagger)
    -\sum_{t\in E}\max_a\mu_t(a;C_\ell)}_{\text{cache gap}}
  +T\rho_T,
\end{align*}
where the sum $\sum_{t\in E}$ runs over exploitation rounds only. The cache gap decomposes per epoch:
\[
\sum_\ell\bigl[H_\ell\,F(C^\dagger)
-\sum_{t\in\mathrm{exploit}\,\ell}\max_a\mu_t(a;C_\ell)\bigr].
\]
Each per-epoch term splits as
\begin{align*}
&H_\ell[F(C^\dagger)-F(C_\ell)]\\
&\quad+\bigl[H_\ell\,F(C_\ell)
-\textstyle\sum_{t\in\mathrm{exploit}\,\ell}
\max_a\mu_t(a;C_\ell)\bigr],
\end{align*}
where the first part is the \emph{expected} cache gap and the second is a realized-vs-expected deviation (controlled in absolute value via Hoeffding). The $T\rho_T=\mathcal{O}(\bar B\sqrt{KT\log(eN/K)})$ benchmark concentration is absorbed into the cache concentration term below.

\textbf{Routing:} Over exploitation rounds, Lemma~\ref{lem:routing} gives
\[
2\beta_T\sqrt{2NdT_E\log(1{+}T/(d\lambda))}
\le 2\beta_T\sqrt{2NdT\log(1{+}T/(d\lambda))}.
\]

\textbf{Bootstrap phase ($\ell<\ell_\star$):} For $\ell<\ell_\star:=\max(\ell_0,\ell_{\mathrm{cache}})$, either the bonus condition or the cache separability condition may fail, so we cannot apply the refined per-epoch cache-gap analysis. The routing loss during these epochs is already covered by Lemma~\ref{lem:routing} above; the remaining \emph{cache gap} is at most $F(C^\dagger)-\max_a\mu_t(a;C_\ell)\le 2$ per round (since $F(C^\dagger)\le 1$, and $|C_\ell|\ge 1$ guarantees a cached arm with reward $\ge -1$, so $\max_a\mu_t\ge -1$), contributing $\sum_{\ell=0}^{\ell_\star-1}2\cdot H_\ell \le 2\cdot 2^{\ell_\star}\le 2^{\ell_\star}\bar B$ to the total. After $\ell_\star$, Lemma~\ref{lem:bootstrap} ensures the bonus condition $\beta_t s_{a,t}<\Delta_{\min}/2$ holds for all arms, so Lemma~\ref{lem:selection} guarantees that each hot arm $a$ is selected on its certified region $X_a$; this suffices for Lemma~\ref{lem:explcount} to give geometric growth of informative plays. Lemma~\ref{lem:cachebt} ensures $C_\ell\subseteq\mathcal{M}_{\mathrm{hot}}$.

\textbf{Cache gap at epoch $\ell \ge \ell_\star$:} By Lemma~\ref{lem:bootstrap}, the bonus condition holds for all arms. By Lemma~\ref{lem:selection}, each hot arm is correctly selected on its region $X_a$, and Lemma~\ref{lem:explcount} gives play counts for previous epochs. By Lemma~\ref{lem:cachebt}, $C_\ell\subseteq\mathcal{M}_{\mathrm{hot}}$. Since both $C^\dagger$ and $C_\ell$ lie in $\mathcal{M}_{\mathrm{hot}}$ (the former by cache separability, the latter by Lemma~\ref{lem:cachebt}), Lemma~\ref{lem:cachebt}(ii) ensures \emph{estimated} hot dominance: $\max_a\hat\mu(a;S;x)=\max_{a\in\mathcal{M}_{\mathrm{hot}}}\hat\mu(a;S;x)$ for all hot caches $S$ (because the all-arm estimation error is below $\Delta_{\mathrm{dom}}/2$, so the true dominance gap $\Delta_{\mathrm{dom}}$ cannot be bridged by estimation noise). Similarly, true hot dominance gives $\max_a\mu(a;S;x)=\max_{a\in\mathcal{M}_{\mathrm{hot}}}\mu(a;S;x)$. Therefore, for every nonempty hot cache $S\subseteq\mathcal{M}_{\mathrm{hot}}$,
\[
  \hat F_n(S)=\hat F_n^{\mathrm{hot}}(S),\qquad
  F(S)=F_{\mathrm{hot}}(S),
\]
where $F_{\mathrm{hot}}(S):=\mathbb{E}_x[\max_{a\in\mathcal{M}_{\mathrm{hot}}} \mu(a;S;x)]$ and $\hat F_n^{\mathrm{hot}}$ is its empirical counterpart. In particular, this holds for $S=C^\dagger$ and $S=C_\ell$ (both of size $K\ge 1$ by definition). Lemma~\ref{lem:cacheid} (applied to the hot-restricted empirical utility with $n=n_\ell+Nm_\ell$ observations and $|\mathcal{M}_{\mathrm{hot}}|\le N$ arms, so the same bound remains valid) gives
\begin{align*}
&F(C^\dagger)-F(C_\ell)
=F_{\mathrm{hot}}(C^\dagger)-F_{\mathrm{hot}}(C_\ell)\\
&\quad\le 2\varepsilon_\ell
+2\bar B\sqrt{\tfrac{2(K\log(eN/K)+\log(12L/\delta))}{n_\ell+Nm_\ell}},
\end{align*}
where $\varepsilon_\ell:=\max_{a\in\mathcal{M}_{\mathrm{hot}}} \|\hat\theta_{a,\ell}^{\mathrm{pre}}-\theta_a^\star\|_2 \le c_1/\sqrt{p_{\min}\cdot 2^{\ell-1}}$ (Lemma~\ref{lem:l2}).

\textbf{Realized vs.\ expected:} $C_\ell$ is determined by the history up to round $n_\ell+Nm_\ell$; epoch-$\ell$ exploitation contexts are conditionally i.i.d.\ with $|\max_a\mu(a;C_\ell;x_t)|\le\bar B$. By Hoeffding, the per-epoch deviation
\[
\bigl|H_\ell F(C_\ell)
-\textstyle\sum_{t\in\mathrm{exploit}\,\ell}
\max_a\mu_t(a;C_\ell)\bigr|
\le\bar B\sqrt{2H_\ell\log(6L/\delta)}
\]
with probability $\ge 1-\delta/(6L)$. Summing over epochs,
\[
\sum_\ell\bar B\sqrt{2H_\ell\log(6L/\delta)}
\le\bar B\sqrt{2\log(6L/\delta)}\cdot(2{+}\sqrt{2})\sqrt{T}.
\]
Since $K\ge 1$, this is dominated by the cache concentration term and is absorbed into the constant~$c_3$.

\textbf{Summing the estimation term over epochs $\ell\ge\ell_\star$:} The per-epoch estimation contribution is
\[
H_\ell\cdot 2\varepsilon_\ell
=\frac{2^{\ell+1} c_1}{\sqrt{p_{\min}\cdot 2^{\ell-1}}}
=\mathcal{O}\!\left(\sqrt{2^\ell/p_{\min}}\right).
\]
The geometric series $\sum_{\ell=\ell_\star}^{L-1}\sqrt{2^\ell/p_{\min}} \le(2{+}\sqrt{2})\sqrt{T/p_{\min}}$ yields the cache estimation term.

\textbf{Summing the concentration term over epochs $\ell\ge\ell_\star$:} By Lemma~\ref{lem:cacheid} with $\delta'=\delta/(6L)$, the per-epoch concentration contribution to the cache gap is
\[
H_\ell\cdot\bar B\sqrt{\tfrac{2(K\log(eN/K)+\log(12L/\delta))}{n_\ell+Nm_\ell}}.
\]
Since $n_\ell=\sum_{j<\ell}(Nm_j+2^j)\ge\sum_{j<\ell}2^j=2^\ell-1 \ge 2^{\ell-1}$, we have $H_\ell/\sqrt{n_\ell+Nm_\ell}\le\sqrt{2}\cdot\sqrt{2^\ell}$. Summing the geometric series: $\sum_{\ell=\ell_\star}^{L-1}\sqrt{2^\ell} \le(2{+}\sqrt{2})\sqrt{T}$. Combining yields the cache concentration term $c_3\bar B\sqrt{KT\log(eNL/(K\delta))}$, where the $\log L$ factor arises from the per-epoch failure probability $\delta/(6\cdot 2^\ell)$.

\textbf{Switching:} $K\gamma$ per epoch, $L$ epochs.

\textbf{Forced exploration:} Epoch~$\ell$ contributes $Nm_\ell$ rounds at per-round pseudo-regret $\le\bar B$. Total: $\sum_\ell Nm_\ell\bar B=\mathcal{O}(Nd\bar B\log^2 T)$, which is $o(T)$.

\textbf{Collecting all terms.} Combining the six contributions yields the detailed bound:
\begin{align*}
  \mathrm{Regret}(T)
  &\le\underbrace{2\beta_T\sqrt{2NdT\log(1{+}\tfrac{T}{d\lambda})}}_{\text{routing}}
  +\underbrace{\tfrac{c_2(\sigma\sqrt{d\log(NLT/\delta)}+\sqrt\lambda)}
    {\sqrt{\sigma_{\min}\,p_{\min}}}\sqrt{T}}_{\text{cache est.}}\\
  &\quad+\underbrace{c_3\bar B\sqrt{KT\log\tfrac{eNL}{K\delta}}}_{\text{cache conc.}}
  +\underbrace{K\gamma L}_{\text{switch}}
  +\underbrace{\textstyle\sum_\ell Nm_\ell\bar B}_{\text{forced}}
  +\underbrace{2^{\ell_\star}\bar B}_{\text{bootstrap}},
\end{align*}
where $L=\lceil\log_2 T\rceil$ and $\ell_\star=\max(\ell_0,\ell_{\mathrm{cache}})$ with
\[
\ell_0:\; M_{\ell_0} \ge \frac{c\,\beta_T^2}{\sigma_{\min}\,\Delta_{\min}^2},\quad
\ell_{\mathrm{cache}}:\; M_{\ell_{\mathrm{cache}}} \ge
\frac{c\,\beta_T^2}{\sigma_{\min}\,
\min(\Delta_{\mathrm{cache}},\Delta_{\mathrm{dom}})^2}.
\]
For fixed problem parameters, the switching, forced, and bootstrap terms are all $o(\sqrt{T})$, so the leading behavior is $\widetilde{\mathcal{O}}(\sigma d\sqrt{NT}+\bar B\sqrt{KT})$.
\end{proof}

\paragraph{Proof summary.}
The six terms correspond to routing (Lemma~\ref{lem:routing}), cache estimation and concentration (Lemma~\ref{lem:cacheid} per epoch), switching ($K\gamma$ per epoch), forced exploration ($\sum_\ell Nm_\ell$ rounds), and bootstrap ($2^{\ell_\star}$ rounds before bonus and cache conditions hold). Assumption~\ref{ass:diversity} enables the chain: margin $\to$ correct selection (Lemma~\ref{lem:selection}), coverage $\to$ play counts (Lemma~\ref{lem:explcount}), full-dimensionality $\to$ design matrix growth (Lemma~\ref{lem:l2}).

\end{document}